\pdfoutput=1
\documentclass[conference]{IEEEtran}

\usepackage{times}
\usepackage{amsfonts}

\usepackage{graphics} % for pdf, bitmapped graphics files
\usepackage{graphicx}
\usepackage{dblfloatfix}
\usepackage{epsfig} % for postscript graphics files
\usepackage{amsmath} % assumes amsmath package installed
\usepackage{amssymb}  % assumes amsmath package installed
\usepackage{color}

\usepackage{helvet}
\usepackage{courier}
\usepackage{wrapfig}
\usepackage{subfig}
\usepackage{multirow}

\usepackage{mathtools}
\DeclarePairedDelimiter\ceil{\lceil}{\rceil}

\def\imagetop#1{\vtop{\null\hbox{#1}}}

\DeclareMathAlphabet{\mathpzc}{OT1}{pzc}{m}{it}

\newlength{\squeezefactor}
\newlength{\hardsqueezefactor}
\newcommand{\mysection}[1]{\vspace{-0.1\squeezefactor}\section{#1}\vspace{-0.1\squeezefactor}}
\newcommand{\mysubsection}[1]{\vspace{-0.15\squeezefactor}\subsection{#1}\vspace{-0.1\squeezefactor}}
\newcommand{\mycaption}[1]{\vspace{-0.2\squeezefactor}\caption{#1}\vspace{-0.4\squeezefactor}}

\newcommand{\wipe}[1]{}
\definecolor{MyDarkRed}{rgb}{0.5,0,0.1}
\definecolor{MyRed}{rgb}{0.9,0,0.1}
\definecolor{MyDarkGreen}{rgb}{0,0.5,0.1}
\definecolor{blue}{rgb}{0,0,1.0}

\newtheorem{theorem*}{\bf Theorem}

\newtheorem{proposition}{\bf Proposition}

\newtheorem{definition}{\bf Definition}

\newtheorem{proofsketch}{Sketch of Proof}

\newcommand{\argmin}[0]{\mbox{argmin\/}}

\newcommand{\craig}[1]{}

\newcounter{algorithm}
\newenvironment{algorithm}[1]
{ 
  \vspace{1em}
  \refstepcounter{algorithm}
  \begin{tabular}{r|l} % longtable
  \multicolumn{2}{c}{\begin{minipage}[c]{0.9\columnwidth} \textbf{Algorithm~\arabic{algorithm}: {#1}} \end{minipage}} \\
  \hline
}
{ 
  \end{tabular}
}

\newcounter{algorithmline}[algorithm]
\newcommand{\newalgline}{\refstepcounter{algorithmline}\thealgorithmline} %\arabic{algorithmline}
\newcommand{\alglinecontent}[1]{\begin{minipage}{0.9\columnwidth}\vspace{0.05in} \hangindent=0.1in {#1} \end{minipage}}

\begin{document}

% paper title
%\title{A Topological Approach to Object Separation and Caging Using a Cable}
% \title{Approximate Structure Construction by Statistical Robot Swarms Using Harmonic Attractor Dynamics}
\title{Approximate Structure Construction Using \newline Large Statistical Swarms}

% You will get a Paper-ID when submitting a pdf file to the conference system
% \author{--- Paper-ID 143}

\author{
\authorblockN{Subhrajit Bhattacharya}
\authorblockA{Department of Mechanical Engineering and Mechanics \\ Lehigh University\\ Pennsylvania, U.S.A.\\
Email: sub216@lehigh.edu}
%  Paper ID\# 177
}

\maketitle

\begin{abstract}
In this paper we describe a novel local algorithm for large statistical swarms using \emph{harmonic attractor dynamics}, by means of which a swarm can construct harmonics of the environment. This in turn allows the swarm to approximately reconstruct desired structures in the environment.
The robots navigate in a discrete environment, completely free of localization, being able to communicate with other robots in its own discrete cell only, and being able to sense or take reliable action within a disk of radius $r$ around itself.
We present the mathematics that underlie such dynamics and present initial results demonstrating the proposed algorithm.
\end{abstract}

%\IEEEpeerreviewmaketitle

% -------------------------------------------------------------------

\mysection{Introduction}

% Use of statistical mechanics in modeling of swarms of robots or mobile sensors is not new~\cite{Correll2015}
Statistical methods in swarm robotics has been studied extensively~\cite{Correll2015}. However, most of the past research has focused on a diffusion-based model where swarms diffuse in an environment while attaining coverage~\cite{Elamvazhuthi_Berman15} or making estimations~\cite{Ramachandran_Elamvazhuthi_Berman,Lu:Harmonic:estimation}.
Particle dynamics (microscopic model) in such problems constitute of a Gaussian \emph{kernel} around a robot's current location, according to which the robot either re-samples its new location or updates its \emph{weight}.

This paper is a first attempt towards developing local behaviors of robots (kernels) that would allow the construction of structured information at a global scale contained in the robot swarm, and hence the environment. % While with local action and local sensing the robots in a swarm itself implies that 
The key innovation in doing this is to develop new dynamics (\emph{harmonic attractor dynamics} represented by novel kernels) that can manifest itself in weights carried by the robots in the swarm, and hence be able to construct stable patterns in spatial distribution of the weights (the \emph{harmonics} of the environment).
While weighted Monte-Carlo methods have been extensively used in the past~\cite{Fishman-Monte-Carlo,seq:montecarlo,probRob:Thrun}, they have mostly been used for estimation and sensing.

The key assumptions behind this work are:
\begin{enumerate}
 \item \emph{No localization:} There is no global localization and the robots in the swarm do not have odometer that would allow them to reliably infer their locations in the global space. If a robot takes an action to move to a new location, the action can be reliably executed only if the new location is within a radius of $r_a$ of its current location. %There is a small radius 
 \item \emph{Local behavior:} The robots can communicate with other robots in a small neighborhood (within a radius $r_c$) and
 can sense presence of boundaries/obstacles only locally (within a limited radius of $r_b$ around itself).
%  and can take only local noisy action (move to a neighboring location within radius $R_a$) with an error proportional to the size/length of the attempted action (\emph{i.e.}, $\eta r_a$, where $r_a$ is the length of the action and $\eta$ is a proportionality constant). % In absence of localization, 
 \item \emph{No central coordination:} There is no global or centralized coordination that can decide and assign tasks to the robots online. The robots can only have fixed, local, pre-determined behaviors described to them at the beginning.
\end{enumerate}
\footnote{For the purpose of this paper we will choose $r_a = r_b = r$ and $r_c = 1$.}
In this paper we consider a discretized representation of the environment and construct a Markov chain model on it, instead of the usual continuous model of environments traditionally used for statistical swarms. Kernels in such a discrete model constitute of local actions that are taken relative to a robot's current location, and do not, in general, depend on robot's position in the environment (unless the robot senses an obstacle/boundary within a radius of $r_b$).
%that are traditionally used in statistical swarms.
% While discrete repesentation 
While statistical swarms have been used in discrete or Markov chain models of environments~\cite{Açikmeşe:swarm,2014arXiv1403.4134B}, past research in this area has relied on an underlying assumption of global localization of the robots, which we do not require in the present paper. %Our algorithm, on the other hand, do not require localization capabilities of the swarm, and the 

\mysection{Background}

\mysubsection{Markov Chain}

In this paper we will consider a discrete representation of an environment. For illustration we will use a line segment environment discretized into $n$ ``\emph{cells}'' (Figure~\subref*{fig:1-d-cells}). However the theory generalizes naturally to planar environments as well.

\begin{figure*}
\centering
\begin{tabular}{cc}
  
  \begin{tabular}{c}
  
    \subfloat[Markov chain model for a robot navigating along a line segment of length $5$, discretized into five segments of unit length, each corresponding to a state.]{\fbox{\includegraphics[width=0.4\textwidth, trim=0 0 0 0, clip=true]{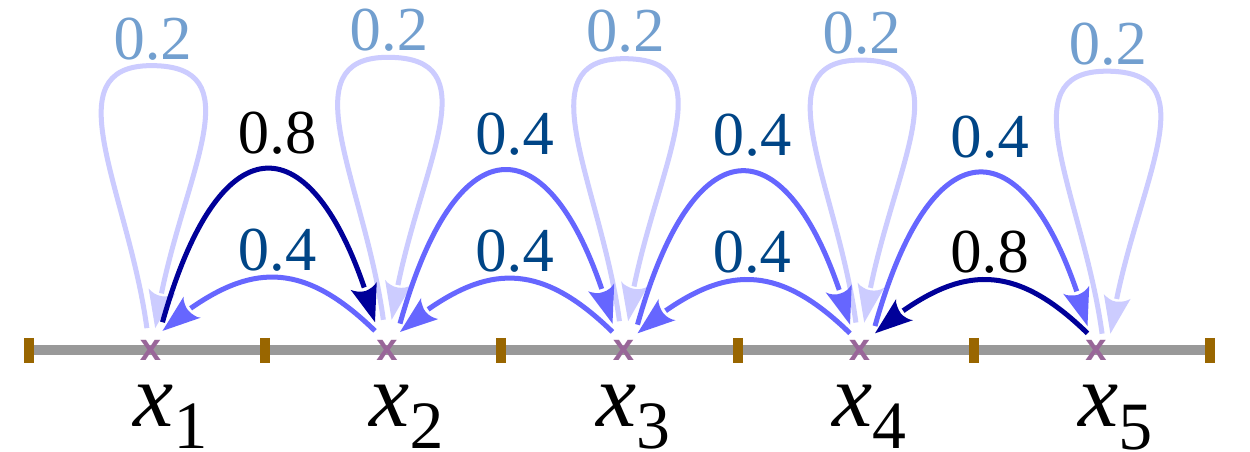}} \label{fig:1-d-cells}} \\
    
    \subfloat[Corresponding \emph{transition probability matrix}.]{\quad $P = \begin{bmatrix} 
% 									  0 & 0.5 & 0 & 0 & 0 \\ 
% 									  1 & 0 & 0.5 & 0 & 0 \\ 
% 									  0 & 0.5 & 0 & 0.5 & 0 \\ 
% 									  0 & 0 & 0.5 & 0 & 1 \\ 
% 									  0 & 0 & 0 & 0.5 & 0
									  0.2 & 0.4 & 0 & 0 & 0 \\ 
									  0.8 & 0.2 & 0.4 & 0 & 0 \\ 
									  0 & 0.4 & 0.2 & 0.4 & 0 \\ 
									  0 & 0 & 0.4 & 0.2 & 0.8 \\ 
									  0 & 0 & 0 & 0.4 & 0.2
									\end{bmatrix}$ \quad  \label{fig:transition-prob-example}} \\
									
	\subfloat[Eigenvalues and eigenvectors of $P$. Notice that the eigenvalue with the highest magnitude is $\lambda_1 = 1$, and the corresponding eigenvector is the only one that has all positive entries.]{ \small $\begin{array}{ll}
			\lambda_1 = 1.00, & \pi_1 = [0.12, 0.25, 0.25, 0.25, 0.13]^T \\ 
			\lambda_2 = 0.77, & \pi_2 = [-0.19, -0.27, 0.00, 0.27, 0.19]^T \\ 
			\lambda_3 = -0.60, & \pi_3 = [-0.12, 0.25, -0.25, 0.25, -0.13]^T \\ 
			\lambda_4 = -0.37, & \pi_4 = [-0.19, 0.27, 0.00, -0.27, 0.19]^T \\ 
			\lambda_5 = 0.20, & \pi_5 = [0.19, -0.00, -0.38, 0.00, 0.19]^T \\
		      \end{array}$ }
  \end{tabular}
  
  \begin{tabular}{c}
   \subfloat[Eigenvector $\pi_1$.]{\includegraphics[width=0.25\textwidth,height=0.17\textwidth, trim=60 200 60 200, clip=true]{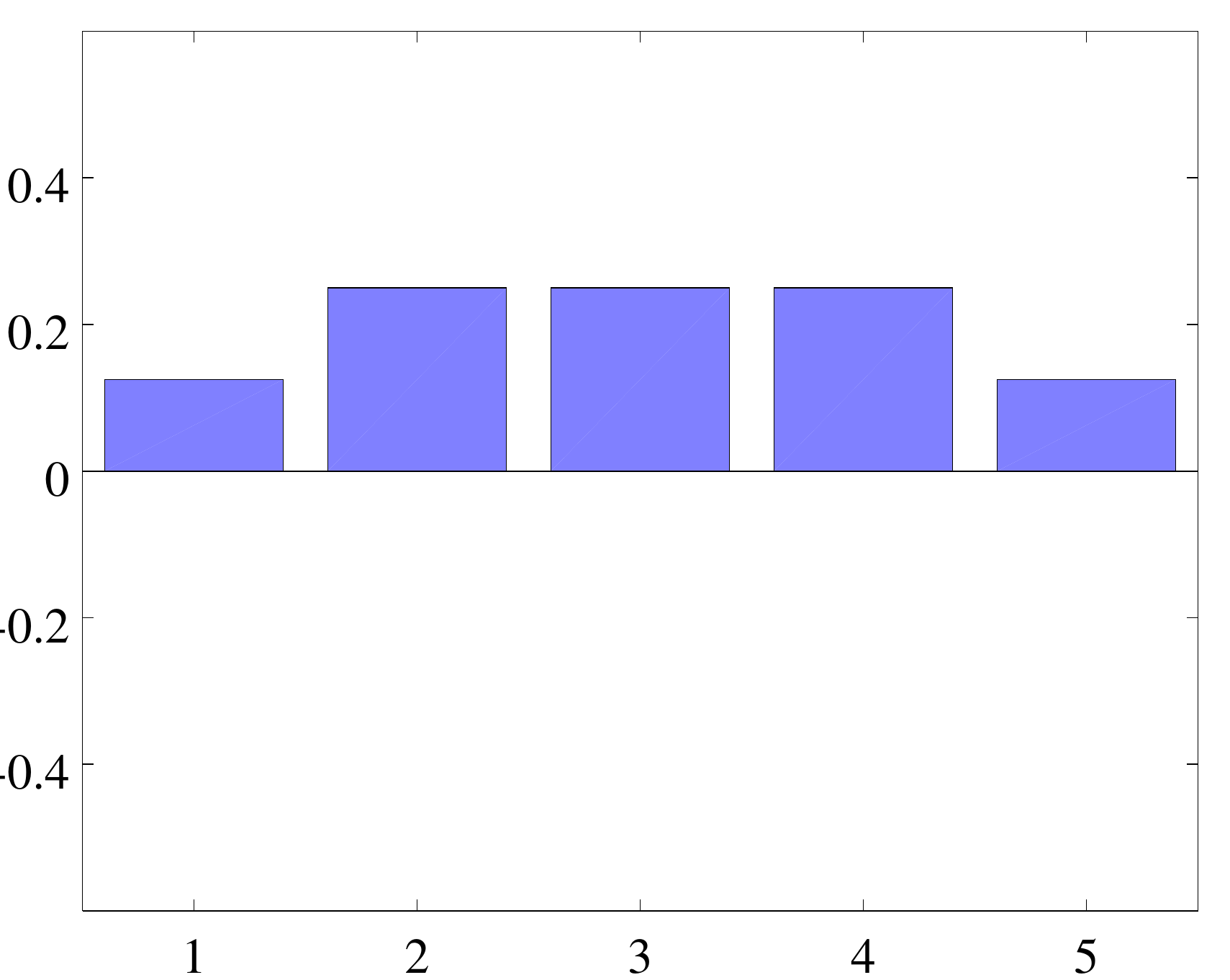} \label{fig:harmonic-1}}  \vspace{-0.1in} \\
   \subfloat[Eigenvector $\pi_2$.]{\includegraphics[width=0.25\textwidth,height=0.17\textwidth, trim=60 200 60 200, clip=true]{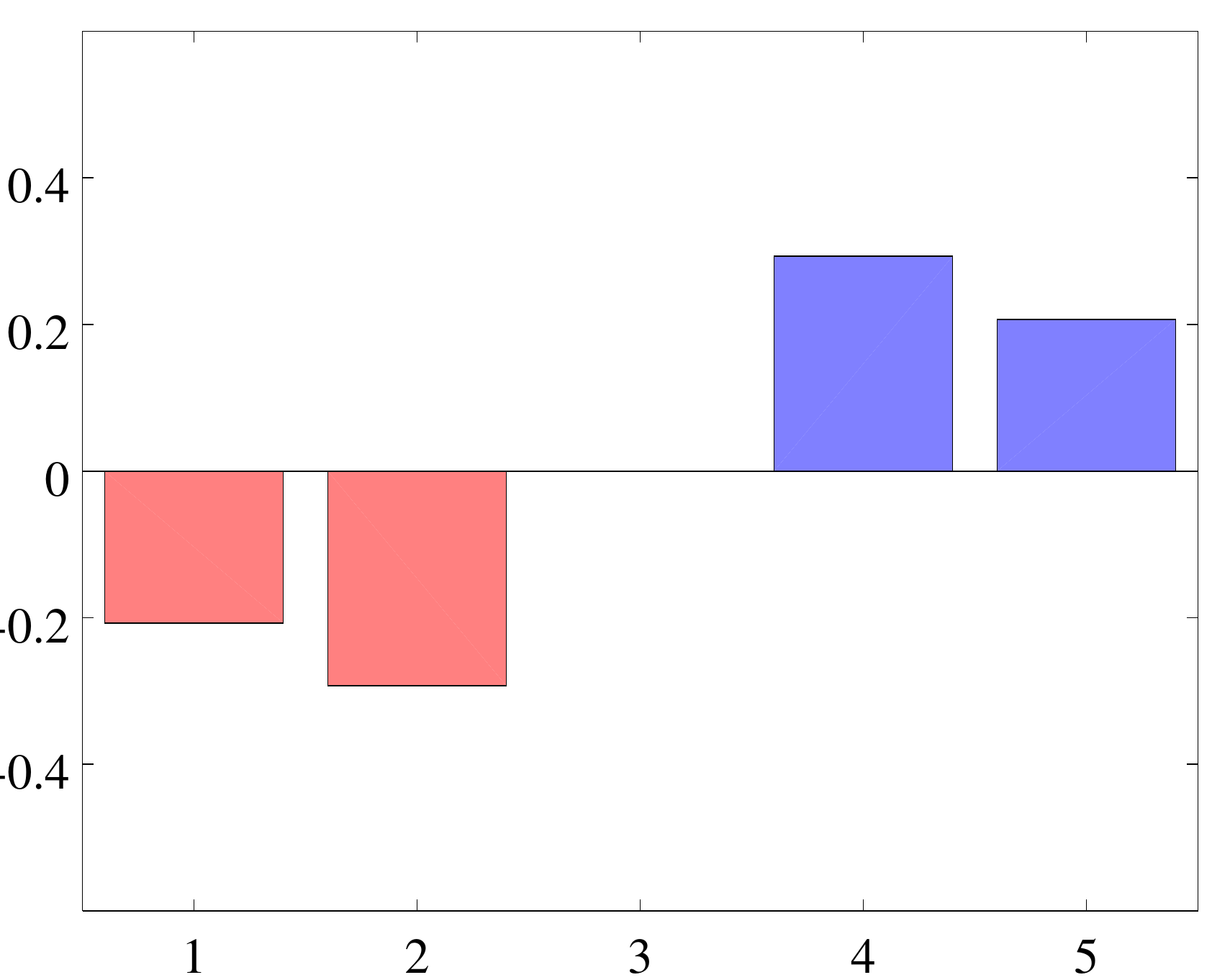} \label{fig:harmonic-2}}  \hspace{-0.1in}
   \subfloat[Eigenvector $\pi_3$.]{\includegraphics[width=0.25\textwidth,height=0.17\textwidth, trim=60 200 60 200, clip=true]{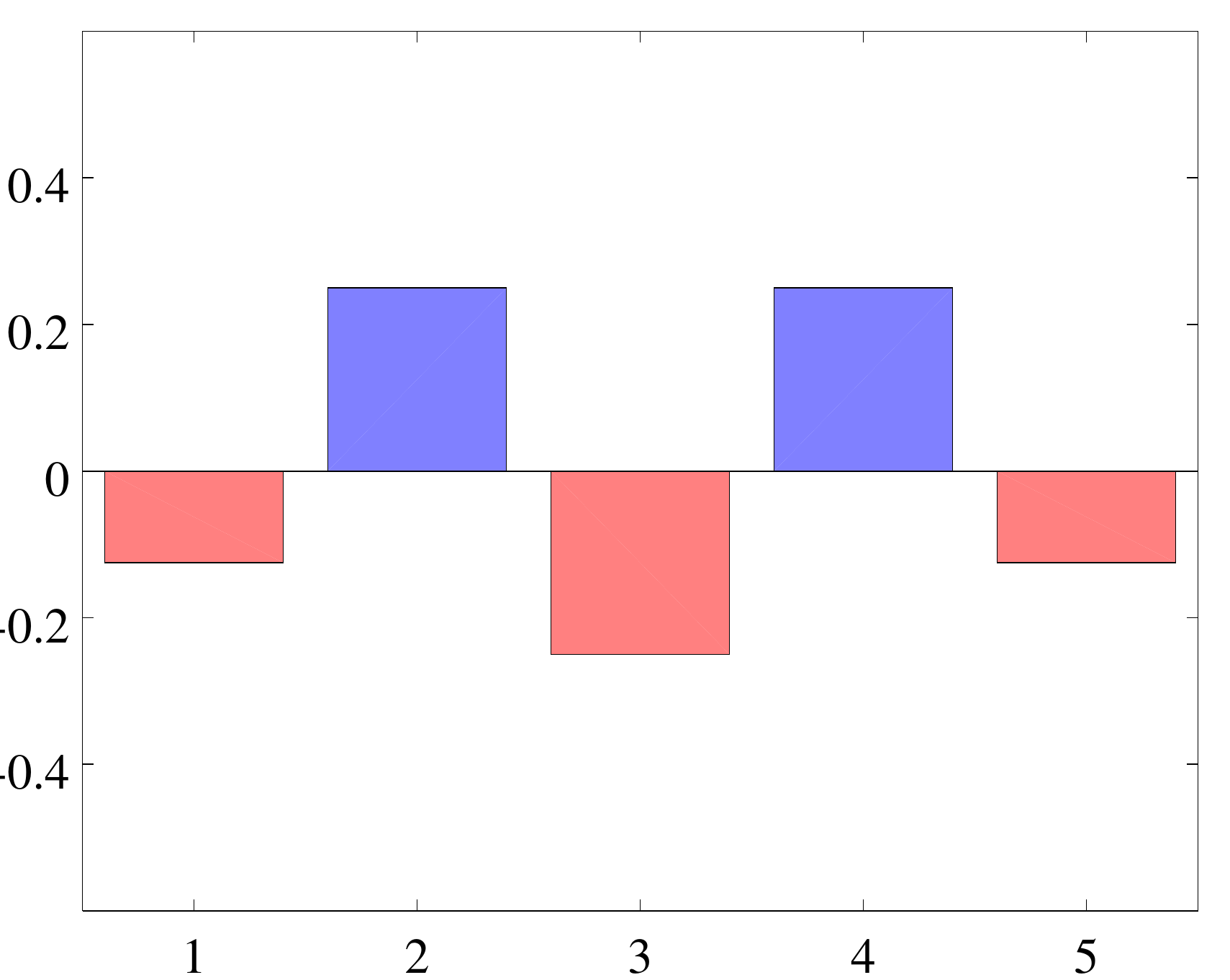} \label{fig:harmonic-3}}  \vspace{-0.1in} \\
   \subfloat[Eigenvector $\pi_4$.]{\includegraphics[width=0.25\textwidth,height=0.17\textwidth, trim=60 200 60 200, clip=true]{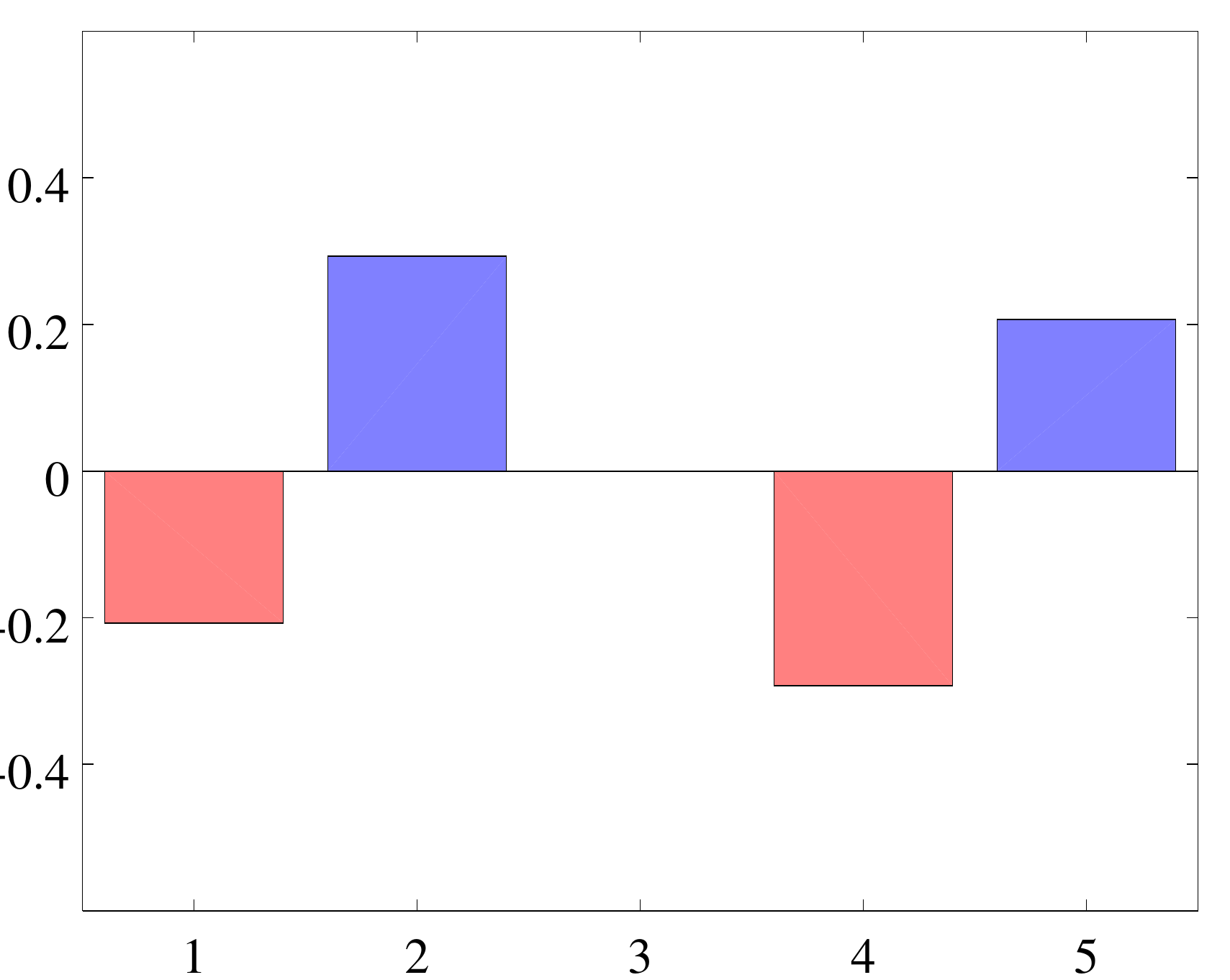} \label{fig:harmonic-4}}  \hspace{-0.1in}
   \subfloat[Eigenvector $\pi_5$.]{\includegraphics[width=0.25\textwidth,height=0.17\textwidth, trim=60 200 60 200, clip=true]{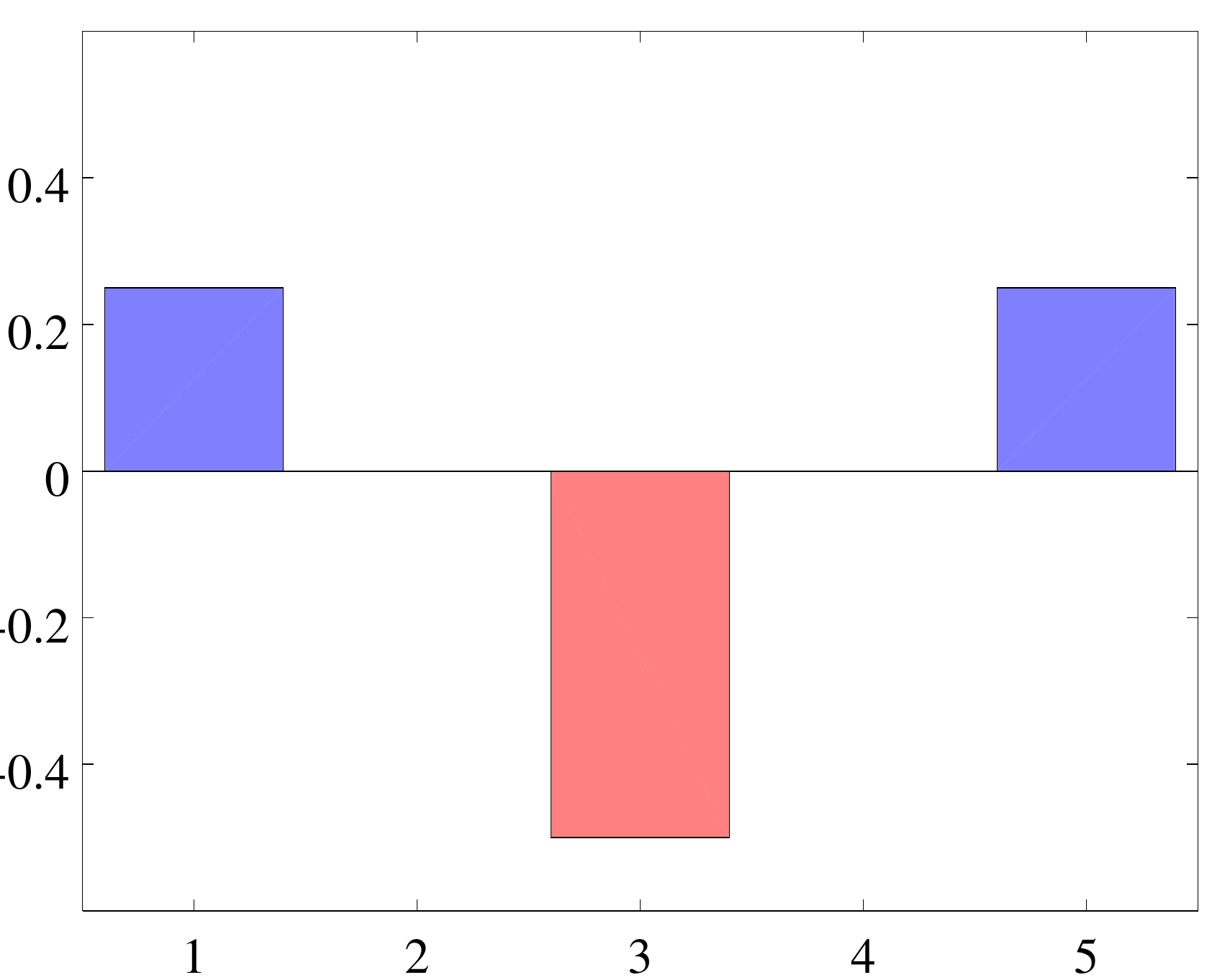} \label{fig:harmonic-5}}
  \end{tabular}
  
\end{tabular} %\vspace{-0.2\squeezefactor}
\mycaption{An illustrative example of a Markov Chain, the transition probability matrix, and its eigen-spaces. (d)-(h): The eigenvectors ($L^1$-normalized) of $P$ constitute \emph{harmonics} of the space with respect to the Markov chain. (The positive bars are colored blue, while the negative bars are colored red, just for visualization.)} \label{fig:1d-example-eigenspaces} %\vspace{-0.4\squeezefactor}
\end{figure*}

We construct a Markov Chain~\cite{meyn2009markov} where every state corresponds to a discrete cell in the environment.
The set of states (cells) is denoted $S = \{x_1, x_2, \cdots, x_n\}$. For a robot residing on a particular cell, it can take an action to move to a neighboring cell, with associated transition probabilities.
% It is assumed that the chain is irreducible -- \emph{i.e.}, 
For example, one can consider the case of a robot navigating on a line segment (Figure~\ref{fig:1-d-cells}), where, if the robot is at $x_i (i\neq 1,n)$, in the next time step it moves to $x_{i-1}$ or $x_{i+1}$ with equal probability of $0.4$, 
% (the probabilities are slightly different near the boundaries, $i\!=\!1,n$, where one of the adjacent states is missing), 
while it remains in $x_i$ with probability of $0.2$.
The states that have only one adjacent state (the boundary states $x_1$ and $x_n$), the transition to the adjacent state happens with a probability of $0.8$, where the robot stays at its previous state with probability $0.2$.
This dynamics can be described using a sparse \emph{transition probability matrix}, $P$ (Figure~\ref{fig:transition-prob-example}), in which the element $P_{ij}$ gives the probability that a robot at $x_j$ transitions to $x_i$ at the next time step. As a consequence, each column of $P$ needs to add up to $1$ (such matrices are referred to as \emph{stochastic matrices}).

A probability distribution over the states is represented as a column vector, $\rho = [\rho_1, \rho_2, \cdots \rho_n]^T$, such that the probability of finding a robot in state $x_i$ is $\rho_i$ (sometimes we will use the equivalent notation $\rho(x_i) = \rho_i$ for convenience).
The probability distribution itself is a dynamical variable that evolves with time (placed in the superscript of $\rho$) according to the dynamics,
\vspace{-0.2\squeezefactor}
\begin{equation} \label{eq:prob-dynamics}
 \rho^{(t\!+\!1)} ~=~ P\,\rho^{(t)}, \qquad t=0,1,2,\cdots
 \vspace{-0.2\squeezefactor}
\end{equation}
% where, the superscript 
\noindent For example, in the example of Figure~\ref{fig:1-d-cells}, if the robot starts at $x_1$ at $t=0$ (\emph{i.e.}, $\rho^{(0)} = [1, 0, 0, \cdots, 0]^T$ -- Figure~\ref{fig:p-dynamics-0}), in the next time step it can be either at $x_1$ or $x_2$ with probability $0.2$ and $0.8$ respectively (Figure~\ref{fig:p-dynamics-1}).
This is validated by the fact that $\rho^{(1)} = P \rho^{(0)} = [0.2, 0.8, 0, 0, \cdots, 0]^T$.

\mysubsection{Steady-state Distribution}

It is easy to observe that the Markov chains as described are \emph{irreducible} (every state can be reached from another state vis some path) and \emph{aperiodic} (for any given positive integer, $\tau$, the robot has a non-zero probability of returning back to its starting state exactly after time $\tau$).
With these assumptions, the convergence of the dynamics \eqref{eq:prob-dynamics} for any initially chosen $\rho^{(0)}$ is guaranteed by the property of a Stochastic matrix that its eigenvalue of maximum magnitude is $1$ and the corresponding eigenspace is one-dimensional (a consequence of the Perron-Frobenius theorem).
This guarantees that the component of $\rho^{(0)}$ in the direction of the eigenvector corresponding to the unit eigenvalue survives, while the components in the direction of any of the other eigenvector (corresponding to eigenvalues with magnitude less than $1$) go to zero due to the dynamics \eqref{eq:prob-dynamics}.

Let's denote the eigenvector of $P$ corresponding to the eigenvalue of $\lambda_1 \!=\! 1$ by $\pi_1$.
The other eigenvalues (with magnitude less than $1$) and corresponding eigenvectors will be denoted by $\lambda_2, \lambda_3, \cdots, \lambda_n$ and $\pi_2, \pi_3, \cdots, \pi_n$ respectively.
% Thus, as $t\rightarrow \infty$, $\rho$ converges to a probability distribution that tells us the probability of finding the robot at a state in the limit. That is,
\footnote{A notational disambiguation: Since $\pi_i$ is itself a vector, we will denote the element of $\pi_i$ corresponding to the location $x_j$ as $\pi_{i,j}$.}
The above discussion implies the following limit result 
\vspace{-0.2\squeezefactor}
\begin{equation} \label{eq:steady-state}
 \lim_{t \rightarrow \infty} \, P^t \, \rho^{(0)} ~=~ \pi_1
 \vspace{-0.2\squeezefactor}
\end{equation}
for any initial probability distribution $\rho^{(0)}$. Note that 
\vspace{-0.2\squeezefactor}
\begin{equation}
 \pi_1 = P \pi_1
 \vspace{-0.2\squeezefactor}
\end{equation}
is the eigenvector of $P$ corresponding to eigenvalue $1$.
This distribution, $\pi_1$, is known as the \emph{steady-state} or \emph{stationary} distribution of the Markov chain.

We introduce the following two terminologies for future reference:
\begin{definition}[Harmonics of the Markov Chain]
 The eigenvectors, $\pi_1,\pi_2,\cdots,\pi_n$, of $P$ will be referred to as \emph{harmonics} of the space with respect to the Markov chain.
 Clearly they form a basis for $\mathbb{R}^n$.
 However, only $\pi_1$ has a \emph{physical interpretation} of being a probability distribution (the steady-state distribution), since all other eigenvectors have negative entries and do not add up to unity.
 
 [P.S.: This terminology is motivated by the natural resemblance of the eigenvectors to harmonics or the shape of standing waves (Figure~\ref{fig:harmonic-1}-\ref{fig:harmonic-5}).]
\end{definition}
\begin{definition}[Attractor of a Dynamics]
 Since the dynamics of equation~\ref{eq:prob-dynamics} always converges to the distribution $\pi_1$ (or a scalar multiple thereof), we say that the dynamics of $P$ is attracting with the unique attractor $\pi_1$.
\end{definition}

\begin{figure}
    \begin{tabular}{c}
%     \centering
    \subfloat[Probability distribution at $t=0$.]{\includegraphics[width=0.24\textwidth,height=0.17\textwidth, trim=60 200 60 200, clip=true]{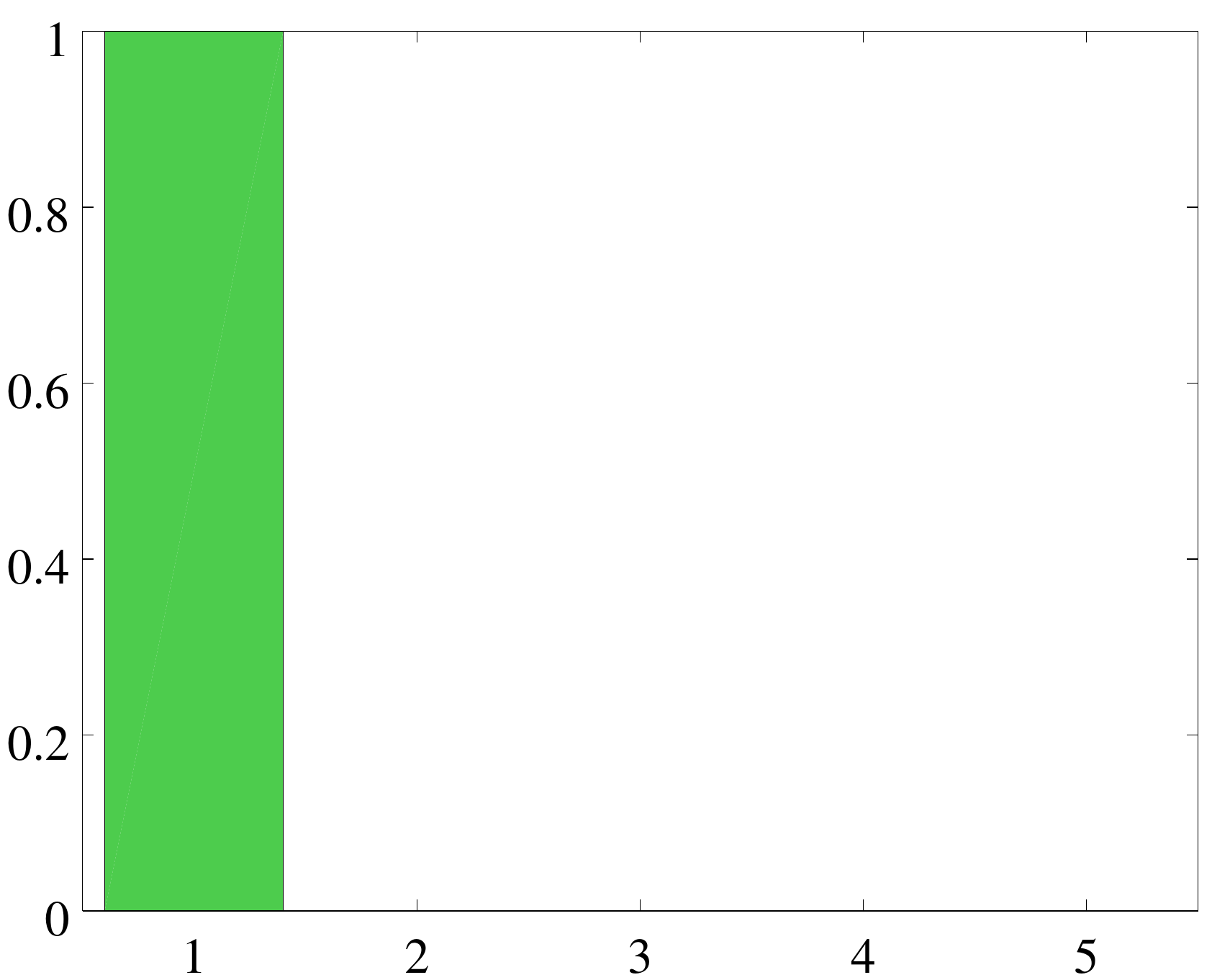} \label{fig:p-dynamics-0}} \hspace{-0.12in}
    \subfloat[$\rho^{(1)}$.]{\includegraphics[width=0.24\textwidth,height=0.17\textwidth, trim=60 200 60 200, clip=true]{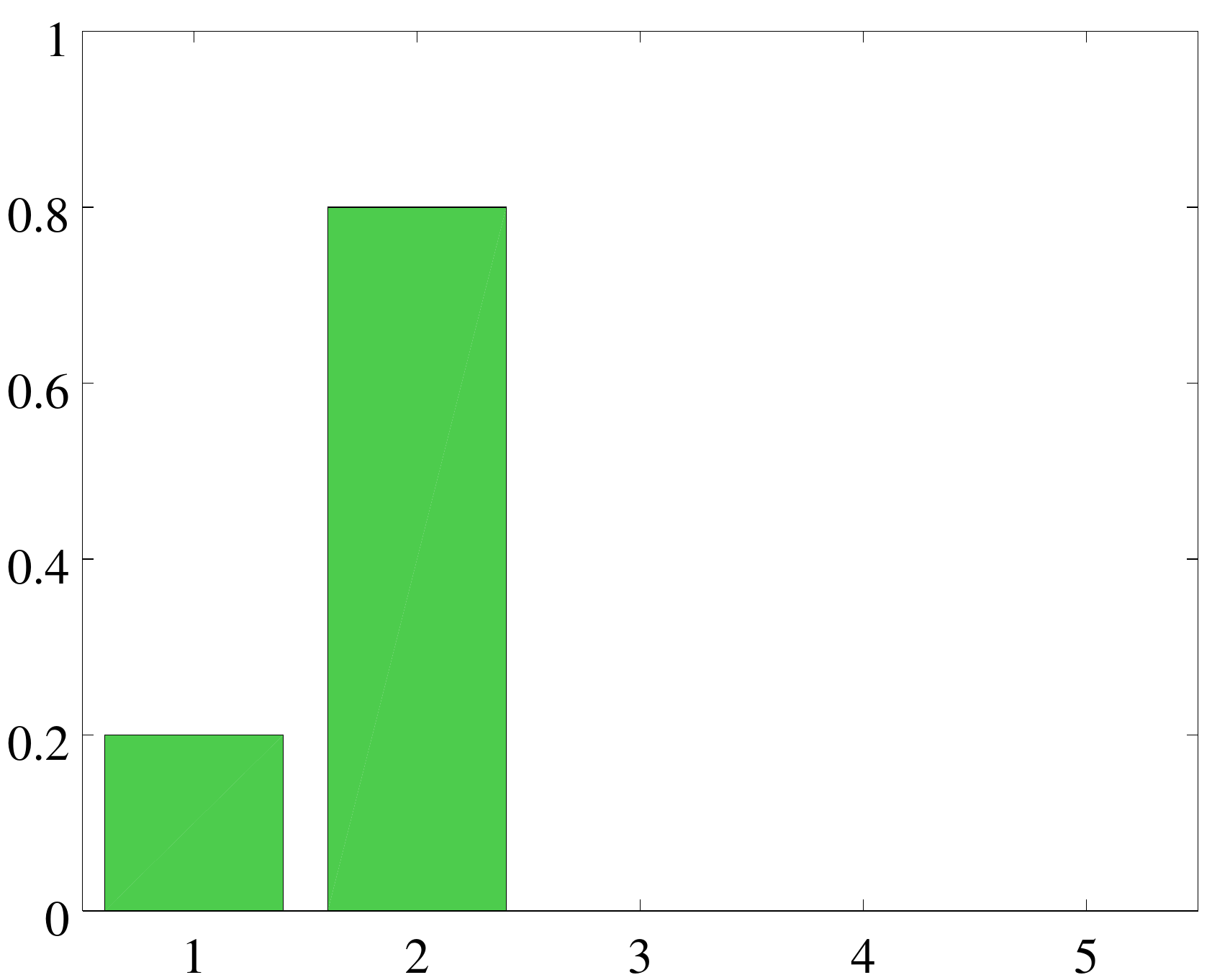} \label{fig:p-dynamics-1}} \vspace{-0.1in} \\
    \subfloat[$\rho^{(3)}$.]{\includegraphics[width=0.24\textwidth,height=0.17\textwidth, trim=60 200 60 200, clip=true]{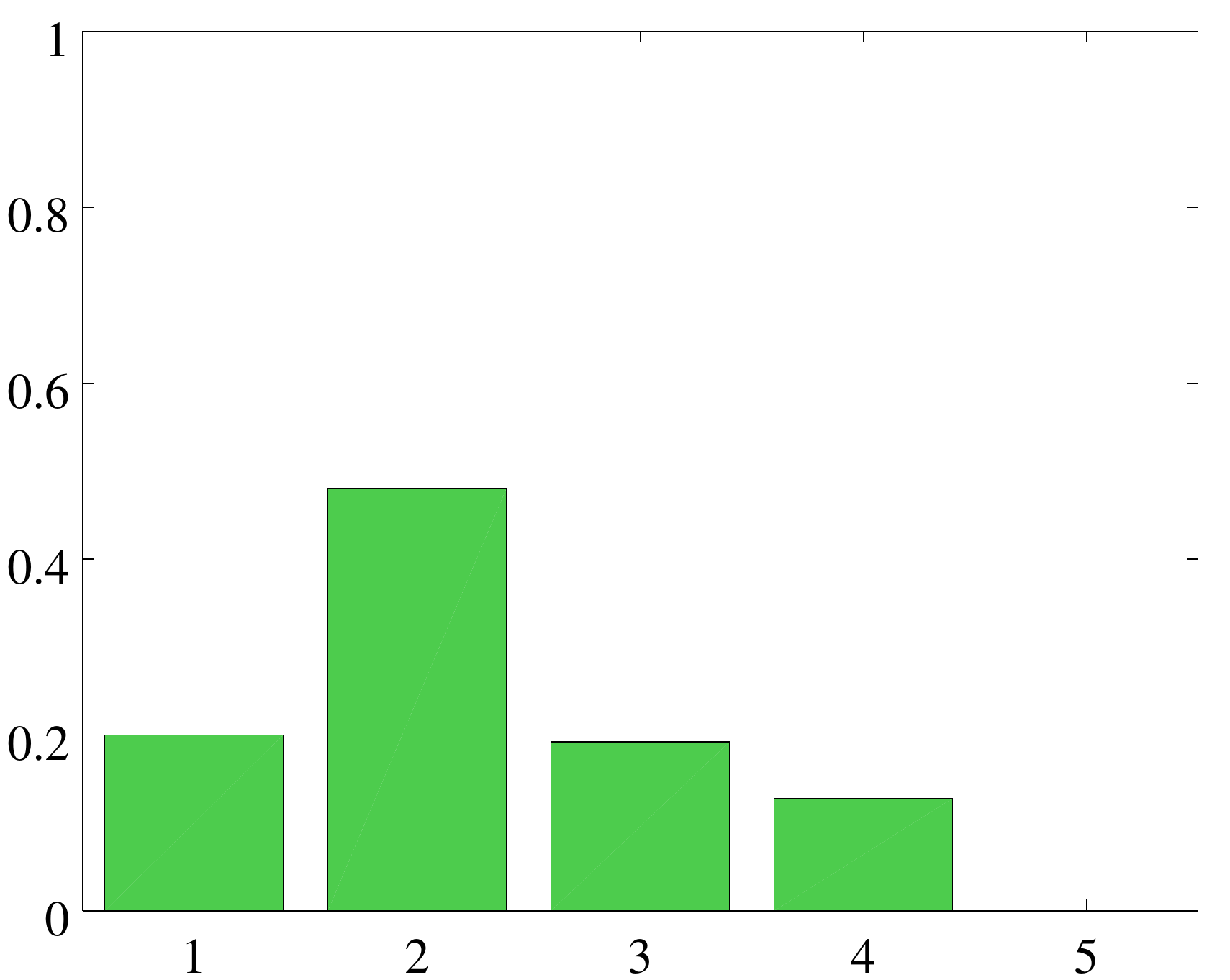} \label{fig:p-dynamics-3}} \hspace{-0.12in}
    \subfloat[$\rho^{(5)}$.]{\includegraphics[width=0.24\textwidth,height=0.17\textwidth, trim=60 200 60 200, clip=true]{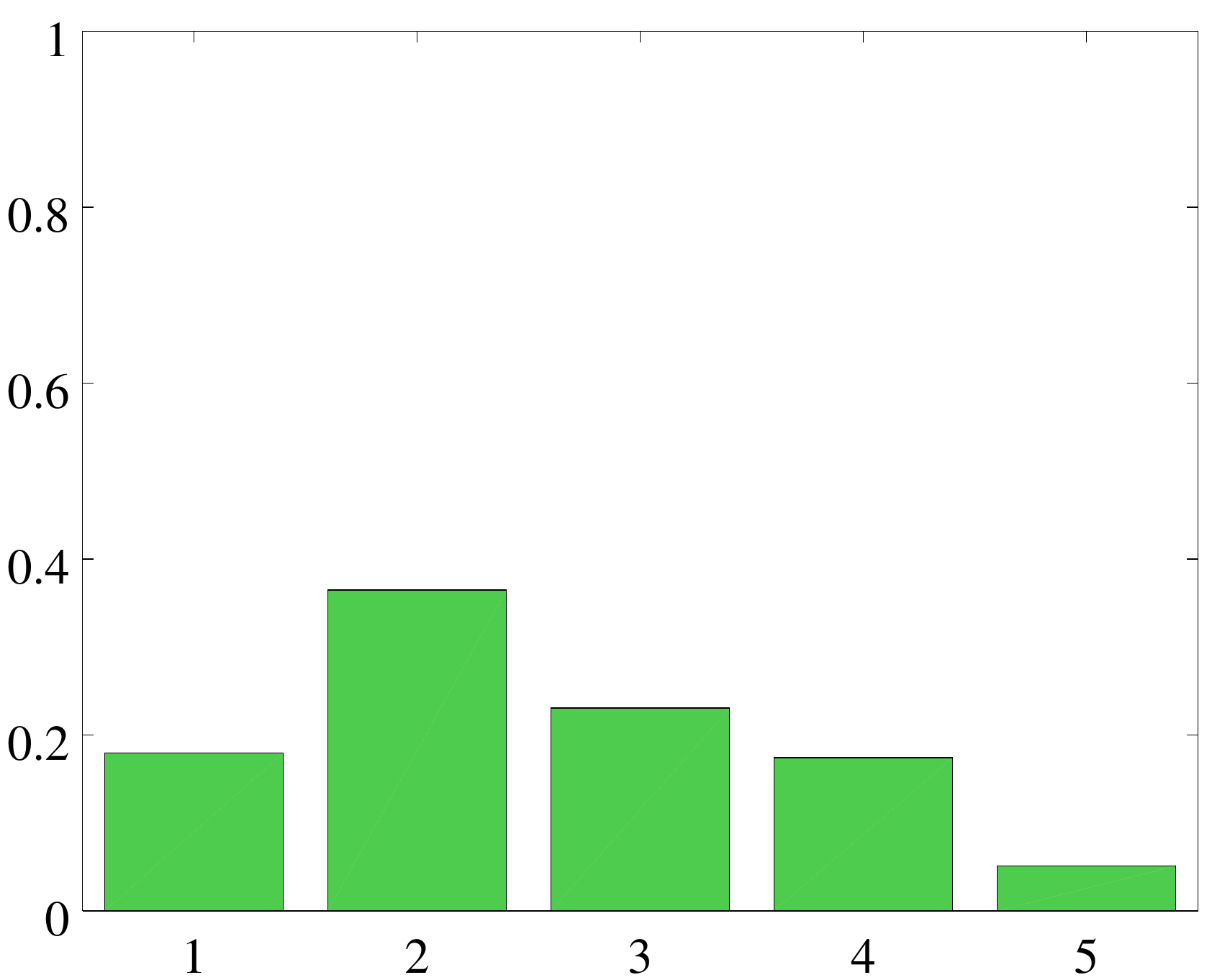} \label{fig:p-dynamics-5}} \vspace{-0.1in} \\
    \subfloat[$\rho^{(10)}$.]{\includegraphics[width=0.24\textwidth,height=0.17\textwidth, trim=60 200 60 200, clip=true]{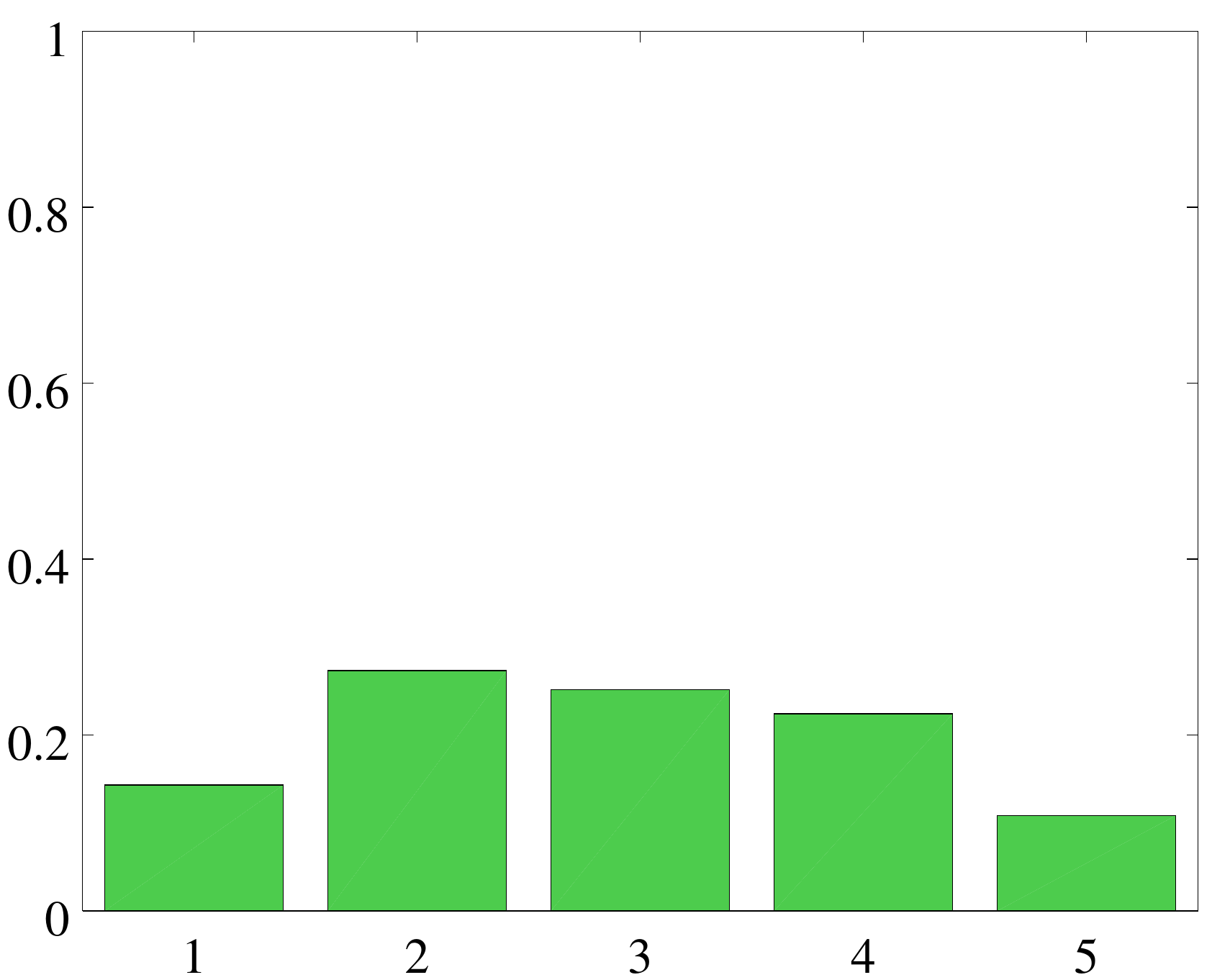} \label{fig:p-dynamics-10}} \hspace{-0.12in}
    \subfloat[$\rho^{(50)}$. Converged to $\pi_1$.]{\includegraphics[width=0.24\textwidth,height=0.17\textwidth, trim=60 200 60 200, clip=true]{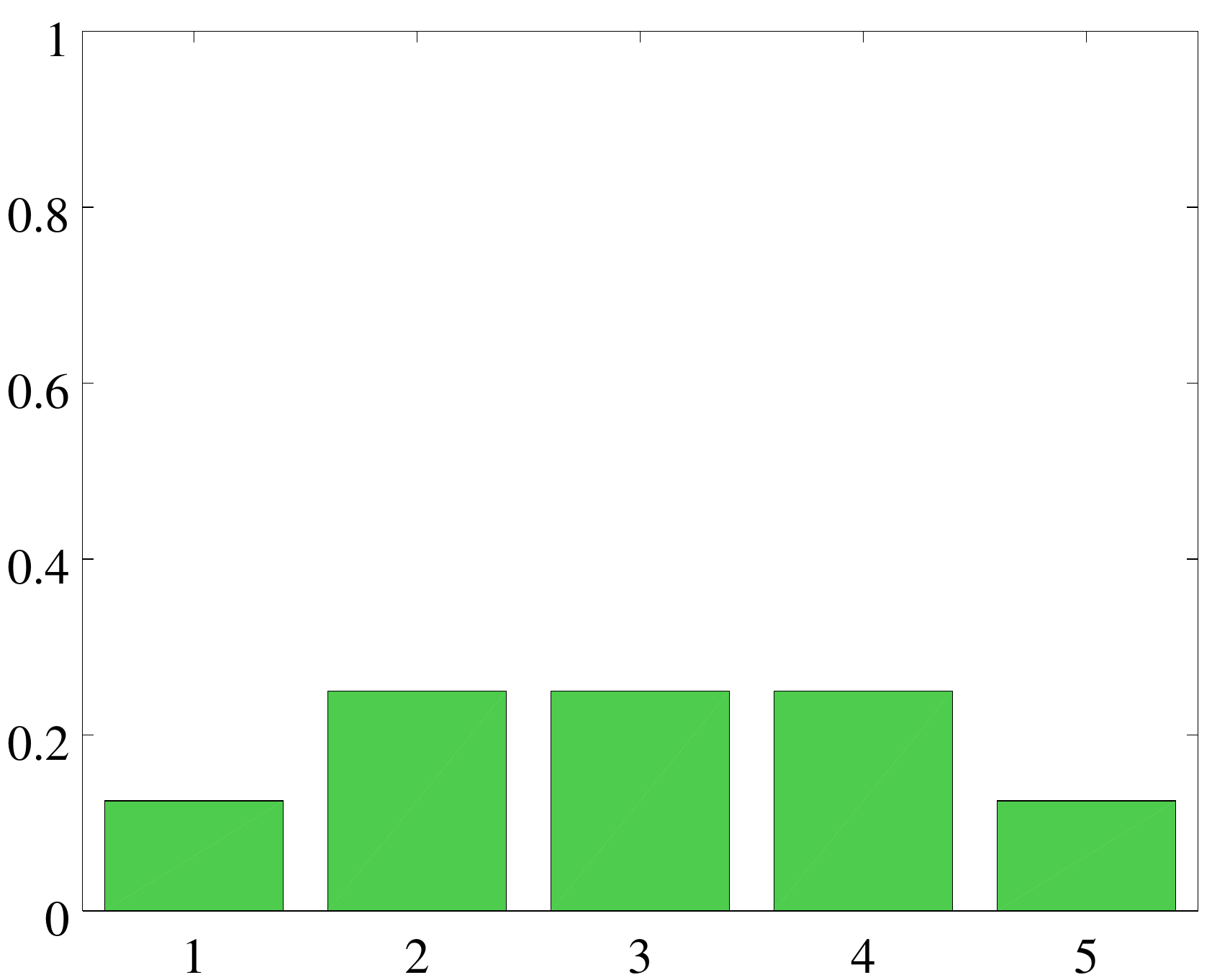} \label{fig:p-dynamics-50}} \vspace{0.1in}
%     \vspace{0.1in}\emph{The dynamics:} The distributions, $\rho^{(t)} = P^t \rho^{(0)}$, as it changes with $t$.
   \end{tabular}
   \mycaption{The dynamics of probability distribution on the Markov chain of Figure~\ref{fig:1d-example-eigenspaces}. The distributions, $\rho^{(t)} = P^t \rho^{(0)}$, changes with $t$ and converges to the eigeenvector $\pi_1$ of $P$.} \label{fig:1d-example-p-dynamics}
\end{figure}

An illustration of the convergence to the limit~\eqref{eq:steady-state} is shown in Figure~\ref{fig:1d-example-p-dynamics}, where, starting with a distribution $\rho^{(0)}$ in which all the particle are located at $x_1$, the distribution eventually converges to $\pi_1$.

% If the limit in \eqref{eq:steady-state} exists, then $\pi$ should be an eigenvector of $P$ with eigenvalue $1$. That is,
% \begin{equation}
%  \pi ~=~ P \, \pi
% \end{equation}
% In fact for a Stohastic matrix it is known~\cite{meyn2009markov} that $1$ is the eigenvalue of maximum magnitude, and the corresponding eigenvector has all positive entries (a consequence of the Perron-Frobenius theorem).
% Since all other eigenvalues have magnitude less than $1$, the 

% Thus the existence and uniqueness of the above limit is guaranteed for any initial distribution $\rho^{(0)}$ due to the Perron-Frobenius theorem:
% \begin{theorem}[Perron-Frobenius]
%  
% \end{theorem}

\mysubsection{The Monte-Carlo Perspective}

The Monte-Carlo method~\cite{Fishman-Monte-Carlo} considers an ensemble of robots that make transitions according to the transition probability matrix, $P$. Unsurprisingly, if the robots in the ensemble move around according to the transition probabilities in $P$, their eventual distribution will converge to that of the steady-state distribution, $\pi_1$. This is illustrated in Figure~\ref{fig:particles-probabilistics-1}-\ref{fig:particles-probabilistics-500} using $20000$ robots navigating on a line segment with $20$ states and the same transition probabilities as in Figure~\ref{fig:1-d-cells}.

% \subsubsection{Local Nature of Monte-Carlo Algorithms}

\vspace{0.02in}
\textit{The Local Nature of the Markov Chain and the Monte-Carlo Algorithm:}
One of the key features of a Markov chain described as above and the corresponding Monte-Carlo algorithm is its \emph{local} nature -- wherever a robot is in its state space, it needs to transition only to a neighboring state, and the probabilities of transition do not depend on absolute location of its current or the neighboring states -- the probabilities depend only on the relative location of the neighboring state with respect to the current state (except for states that are close to the boundaries). This makes a Markov chain description of robot navigation important in absence of localization -- the robot may not know what its current state $x_i$ is, but it can still identify the neighboring states $x_{i-1}$ and $x_{i+1}$ to move to with designated probability that are independent of the absolute position of the robot. The exception is only at the boundaries where the transition probabilities can be different, and hence a robot needs to be able to sense the presence of a boundary and adjust its behavior accordingly.
% , and consequently its Monte-Carlo 

The local nature of the transition probabilities is reflected in the matrix $P$ by the fact that it's a multi-diagonal matrix (a tri-diagonal matrix for the examples in Figure~\ref{fig:transition-prob-example} or \ref{fig:p-matrix-20}) and the columns of $P$ are vertically shifted versions of each other (except for the boundary columns, where, once again, the boundary effects are manifested).
The non-zero entries in each column of $P$ will be referred to as \emph{Kernel} of the dynamics of the particles. Thus, for the matrix in Figure~\ref{fig:transition-prob-example}, the Kernel for particles far from the boundaries consist of $[\cdots,0.4, 0.2, 0.4, \cdots]^T$, where ``$\cdots$'' represent the appropriate number of zeros depending on the particle's position. Of course at the boundaries the Kernels are different ($[0.2,0.8,\cdots]^T$ and $[\cdots,0.8,0.2]^T$ for the $P$ in Figure~\ref{fig:transition-prob-example}).

We introduce the following notation: We assign a differential/relative index of $0$ for the position that represents the current state, and differential indices $-1$ and $1$ to the two positions adjacent to the current state, and so on. Thus, the kernel vector will be represented as $K_{*} = [\cdots, k_{-1}, k_0, k_{1},\cdots]^T$. Thus in the example of \ref{fig:p-matrix-20}, for the kernel $K_{*}$ away from the boundary, $k_0 = 0.2$, $k_{\pm 1} = 0.4$ and $k_{\delta} = 0~\forall |\delta| > 1$.
We define the \emph{radius of the kernel}, $r = \max(\{|\delta| ~\big|~ k_{\delta}\neq 0\})$.
If the state under consideration is a boundary state, $b$, we will write the corresponding kernel vector as $K_b$.

\begin{figure*} 
    \begin{tabular}{cc}
     \imagetop{\begin{tabular}{c}
	\subfloat[The $P$ matrix for $n=20$.]{\includegraphics[width=0.3\textwidth, trim=0 0 0 0, clip=true]{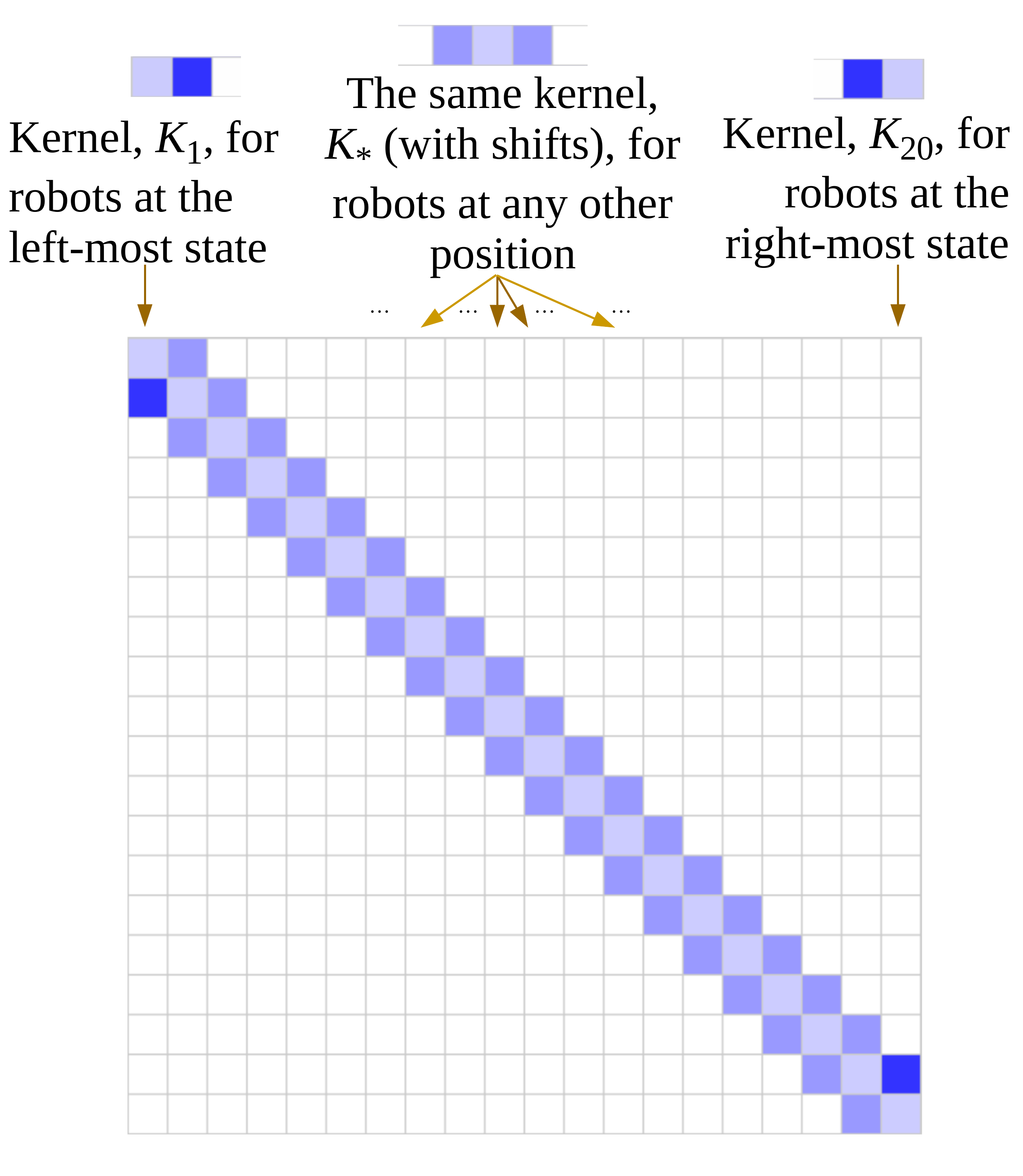} \label{fig:p-matrix-20}} \\
	\subfloat[The corresponding eigenvector, $\pi_1$.]{\includegraphics[width=0.32\textwidth, trim=0 0 0 0, clip=true]{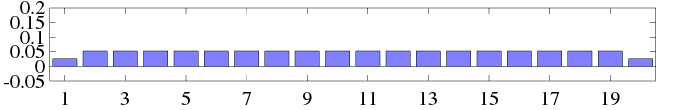} \label{fig:p-pi1-20}}
     \end{tabular} } \hspace{-0.5in}
     & 
     \imagetop{\begin{tabular}{cc}
	\subfloat[Robot distribution at $t=1$.]{\includegraphics[width=0.26\textwidth,height=0.18\textwidth, trim=50 340 50 200, clip=true]{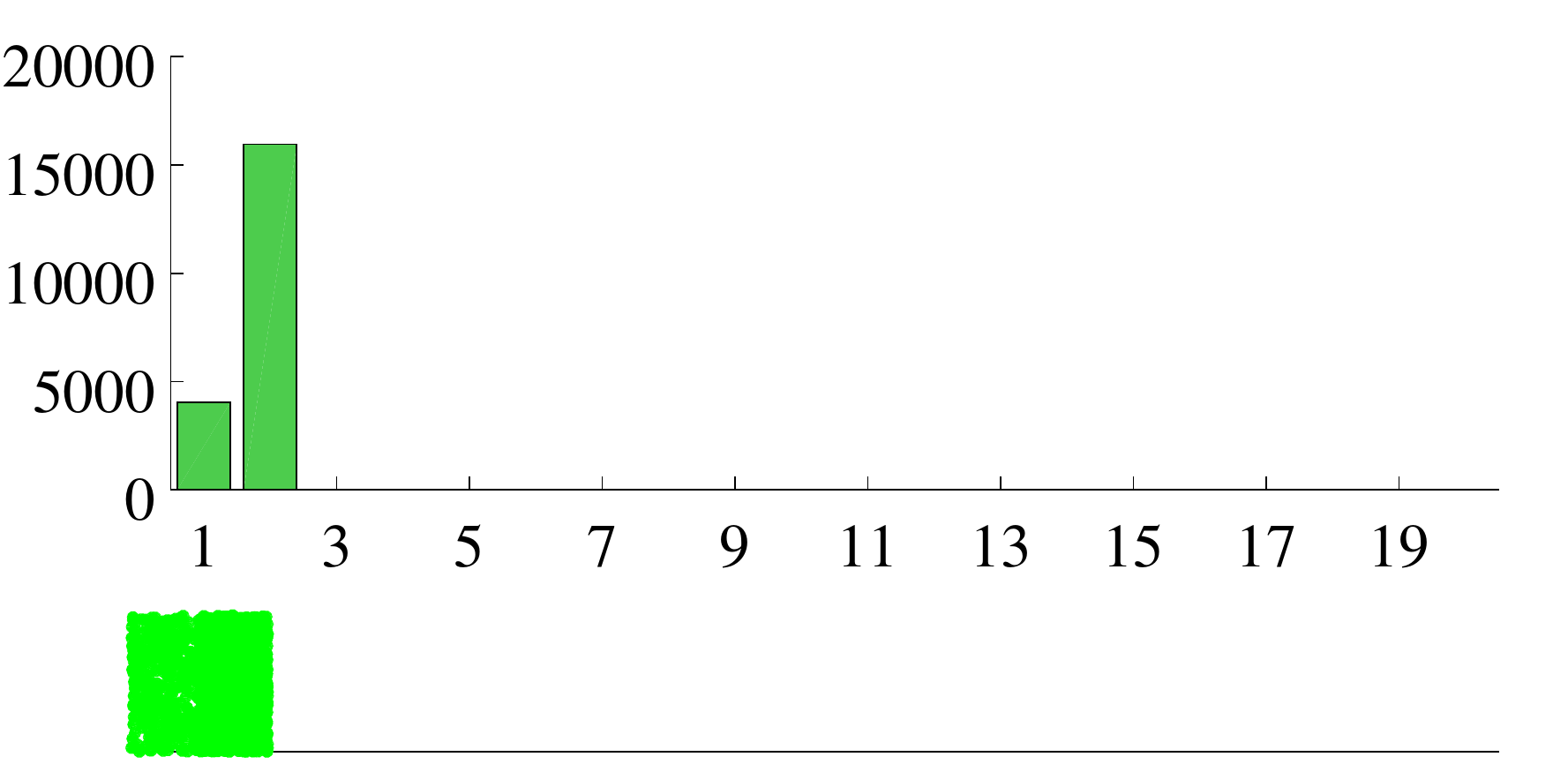} \label{fig:particles-probabilistics-1}} &
	\subfloat[Robot distribution at $t=10$.]{\includegraphics[width=0.26\textwidth,height=0.18\textwidth, trim=50 340 50 200, clip=true]{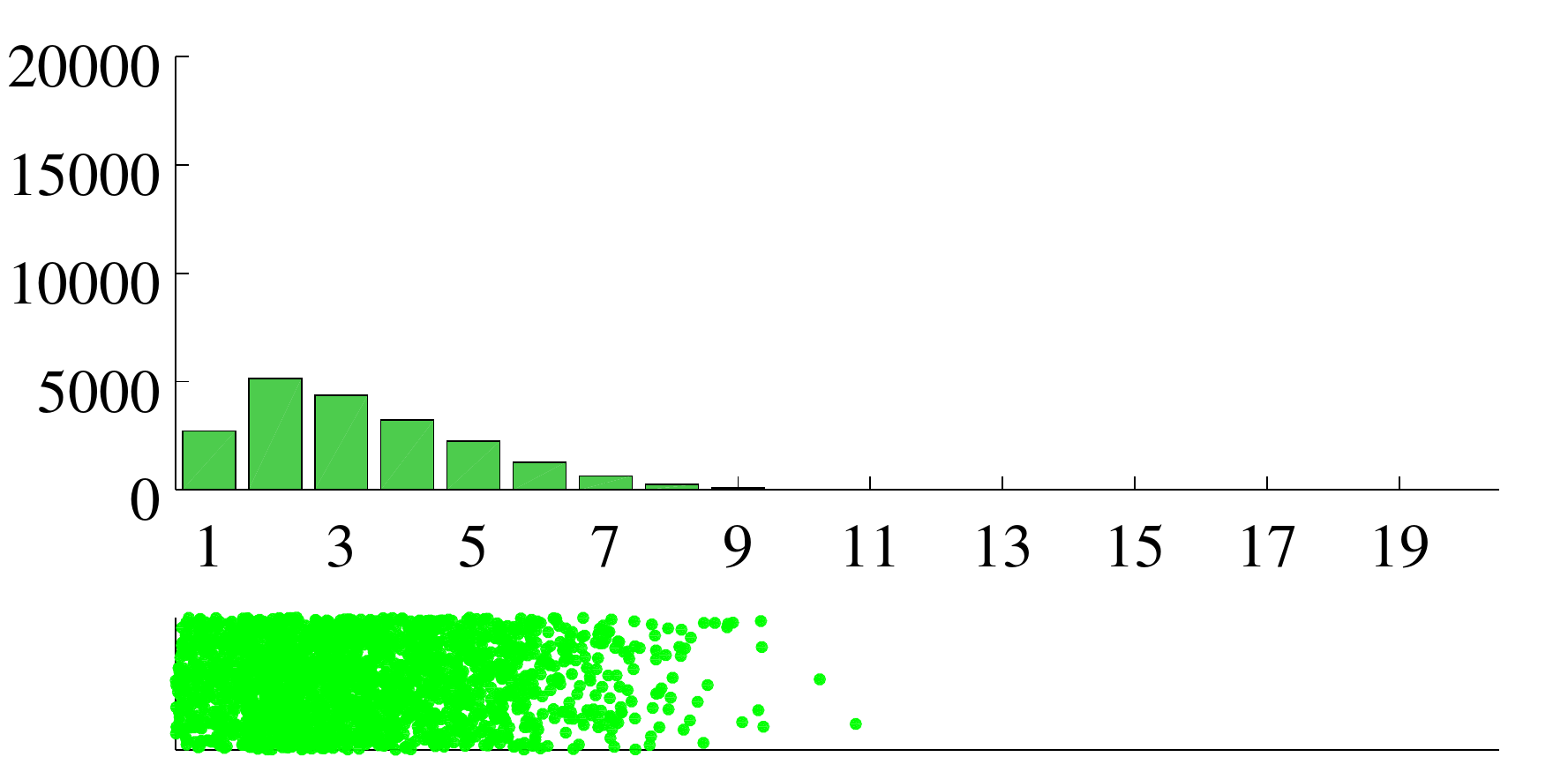} \label{fig:particles-probabilistics-10}} \\
	\subfloat[Robot distribution at $t=100$.]{\includegraphics[width=0.26\textwidth,height=0.18\textwidth, trim=50 340 50 200, clip=true]{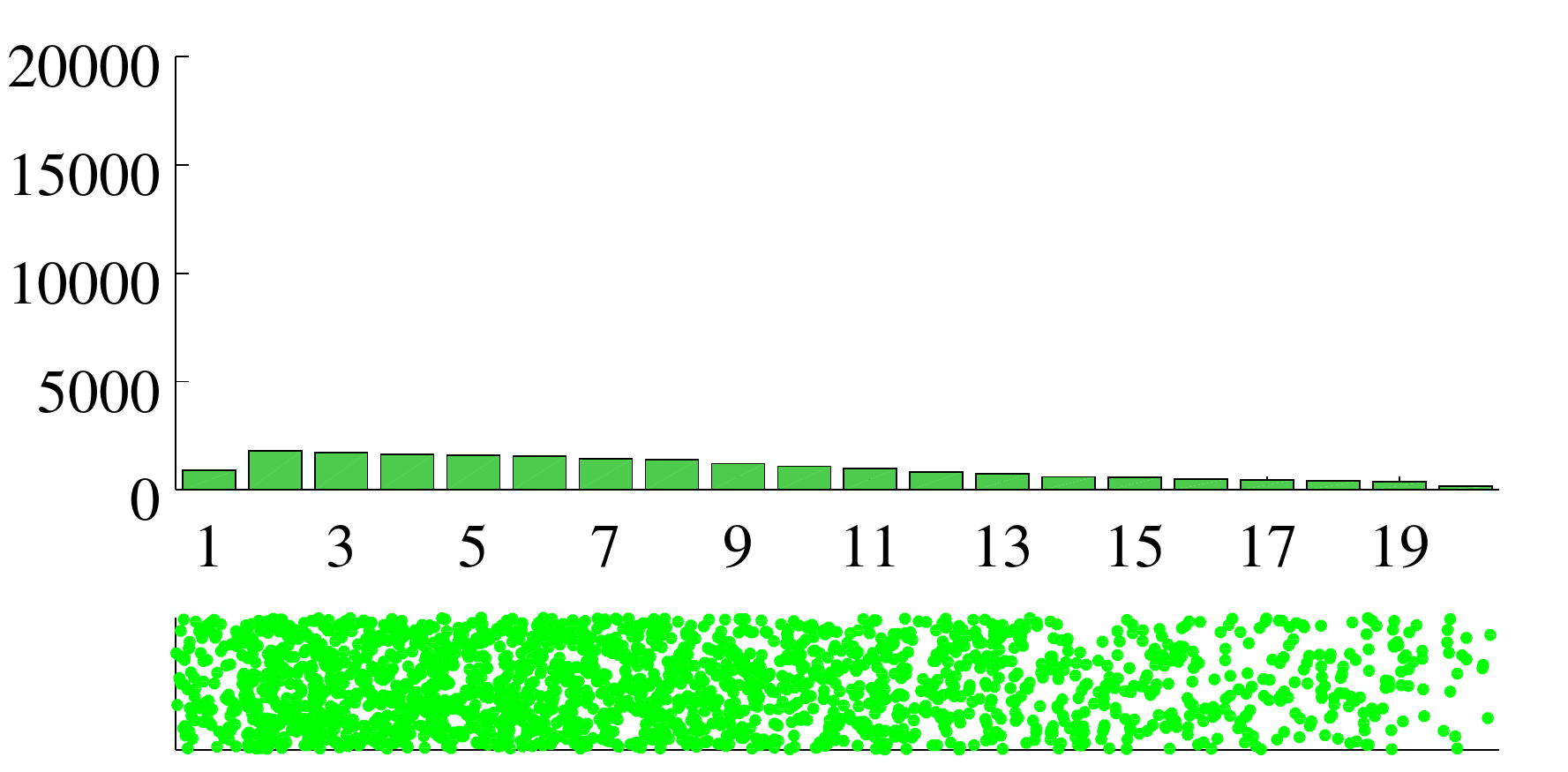} \label{fig:particles-probabilistics-100}} &
	\subfloat[Robot distribution at $t=500$.]{\includegraphics[width=0.26\textwidth,height=0.18\textwidth, trim=50 340 50 200, clip=true]{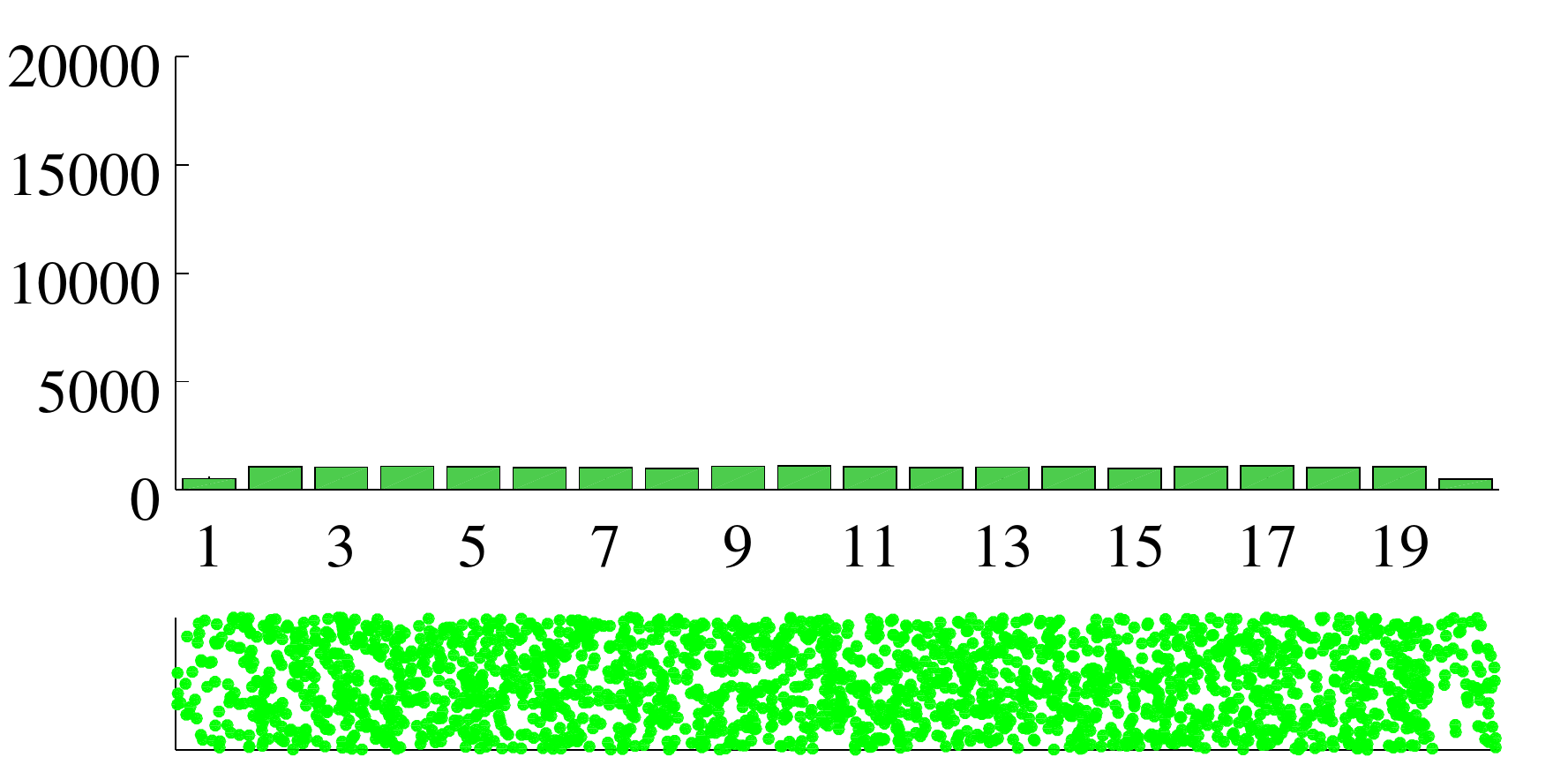} \label{fig:particles-probabilistics-500}} \vspace{0.1in}
        \\ \multicolumn{2}{c}{
	  \begin{minipage}{0.62\textwidth} \small (c-f): Unweighted Monte-Carlo algorithm using $20000$ robots, with robots transitioning according to the transition probability matrix, $P$. The green dots shown below the particle histograms is purely for visualization, with each green dot representing $10$ robots. Note how the distribution (proportion of robots in the different states) converges to $\pi_1$. \end{minipage} }
     \end{tabular} } \\
     \multicolumn{2}{c} {
	\subfloat[Aggregated weights at $t=1$.]{\includegraphics[width=0.26\textwidth,height=0.18\textwidth, trim=50 340 50 233, clip=true]{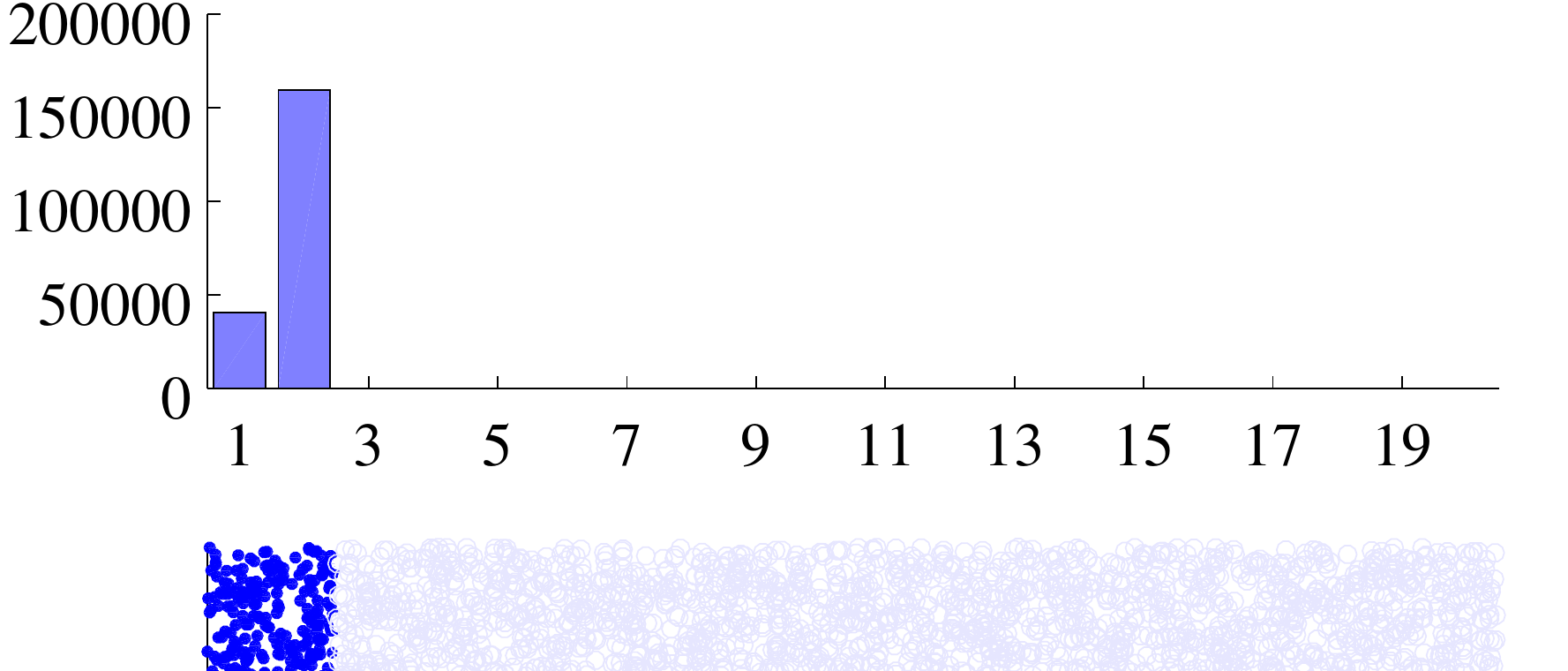} \label{fig:particles-weight-1}} \hspace{-0.19in}
	\subfloat[Aggregated weights at $t=10$.]{\includegraphics[width=0.26\textwidth,height=0.18\textwidth, trim=50 340 50 233, clip=true]{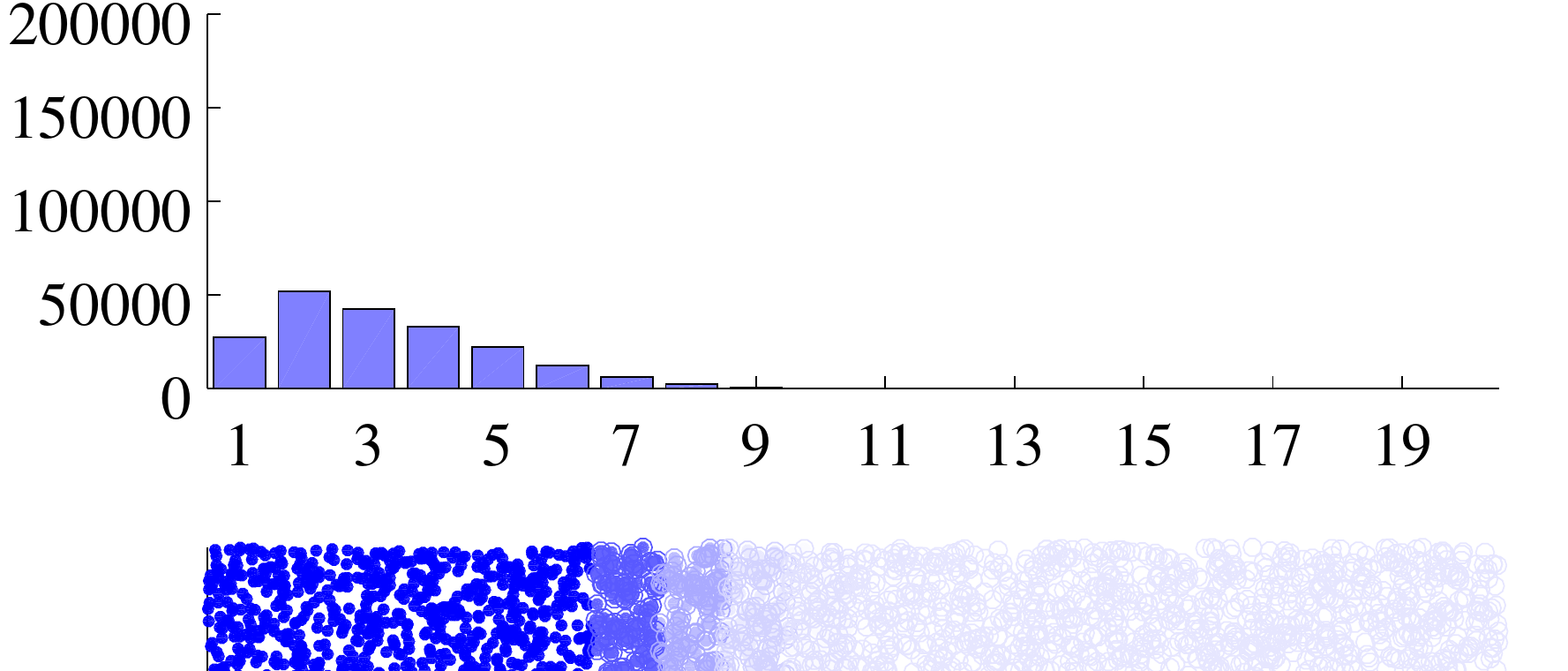} \label{fig:particles-weight-10}} \hspace{-0.19in}
	\subfloat[Aggregated weights at $t=1$.]{\includegraphics[width=0.26\textwidth,height=0.18\textwidth, trim=50 340 50 233, clip=true]{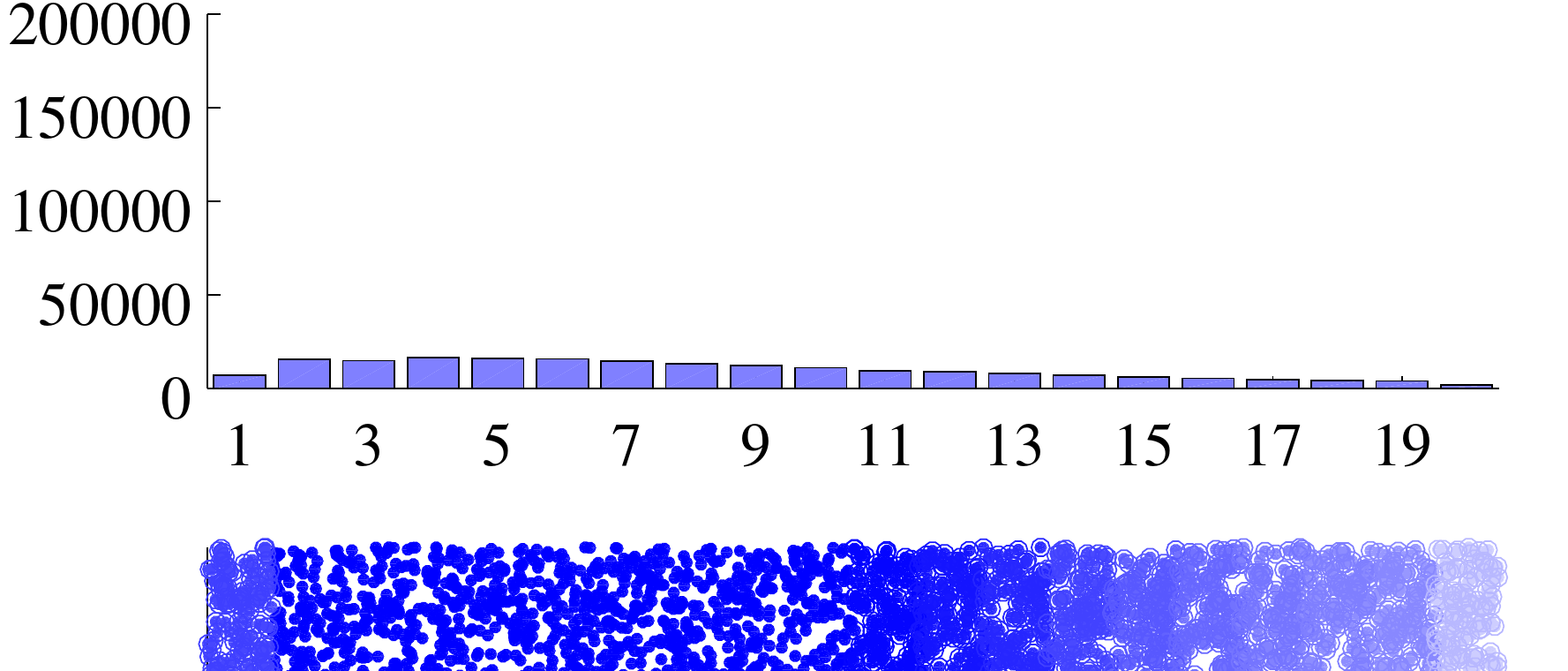} \label{fig:particles-weight-100}} \hspace{-0.19in}
	\subfloat[Aggregated weights at $t=1$.]{\includegraphics[width=0.26\textwidth,height=0.18\textwidth, trim=50 340 50 233, clip=true]{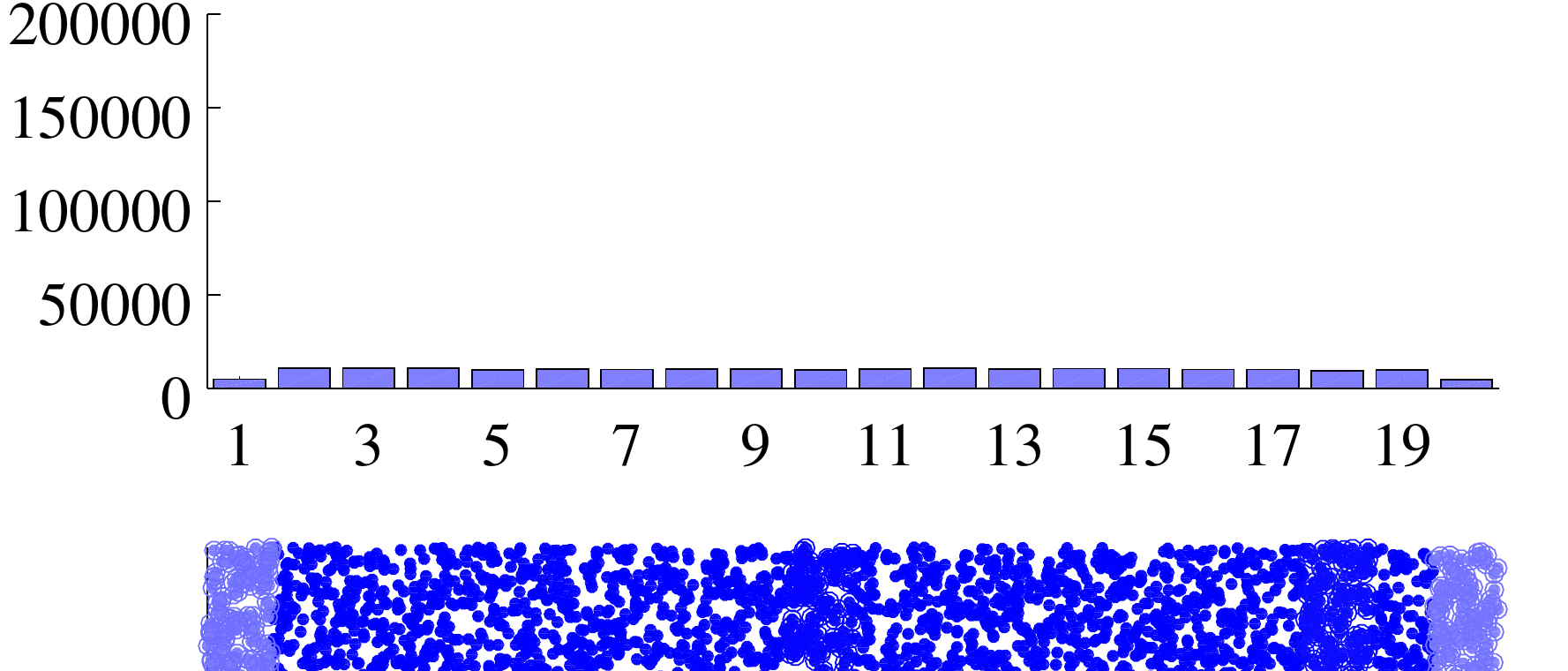} \label{fig:particles-weight-500}} \vspace{0.1in}
     } \\
     \multicolumn{2}{c}{
      \begin{minipage}{1.0\textwidth} \small (g-j): Monte Carlo algorithm using $200000$ robots with weights. Robots transition according to uniform transition probability matrix, $U$. But they update their weights according to transition weight update matrix $W = n P$. The bar plot shows the sum total of weights on the robots in each state (the \emph{aggregated weight} vector $\overline{w}$). Each blue/white dot below represent $100$ robots, and the color of the dots represent the weights (darker is weight close to $1.0$, while lighter is weight close to $0$). At $t=0$ the robots strt at $x_1$, each with a weight of $1.0$. Note how at $t=1$ robots have already distributed themselves uniformly across all the states, although the aggregated weights haven't. The aggregated weights (normalized by the total weight across all states) eventually converge to to $\pi_1$. \end{minipage}
     }
    \end{tabular}
    \vspace{0.1\squeezefactor}\mycaption{Monte-Carlo algorithm for robot swarm navigating on a line segment represented as a Markov Chain with $20$ states. (c-f): The unweighted version, and, (g-j): the weighted version. In both the cases all the robots start at state $x_1$ (left-most state) at $t=0$. In the weighted version the initial weight on each particle is $1$.
%     (c)-(f): Monte Carlo simulation using $20000$ particles with particles taking action (state transition) according to the probability of transition matrix, $P$ (the green dots shown below the particle histograms is purely for visualization, with each green dot representing $10$ particles). Note how the distribution (proportion of particles in the different states) converges to $\pi_1$. 
%     (g)-(j): Monte Carlo simulation using $20000$ particles with associated weights. 
    } \label{fig:1d-markov-p-20}\vspace{0.1\squeezefactor}
\end{figure*}

\subsubsection{Monte-Carlo with Weights}

In the above description of the Monte-Carlo algorithm, the robots distribute themselves into an approximate number density of $g = N \pi_1$ (where $N$ is the number of robots).
Very often in Monte-Carlo methods one assigns weights to the robots as well~\cite{probRob:Thrun,seq:montecarlo}, and updates them based on the transitions that the robots make. This serves as a means of estimating a distribution auxiliary to the density distribution.

In such a setup, one defines a $n\times n$ \emph{transition weight update matrix}, $W$, %similar to the transition probability matrix, $P$, 
with entries $W_{ij}$, which gives the factor by which a particle's weight is to be multiplied when it makes a transition from state $x_j$ to state $x_i$.
% 
% In particular, 
% If the initial density distribution is $\rho^{(0)}$
The following simple observation is vital:

% In particular the following result holds:
\begin{proposition} \label{prop:weight-dynamics}
  If the robots in a Monte-Carlo algorithm transition according to transition probability matrix $Q$, 
  and uses a transition weight update matrix $W$ to update their weights, %has associated weight 
  then %as $N \rightarrow \infty$ (\emph{i.e.} for large number of particles), 
  the vector of aggregated weights in the states, $\overline{w} = [\overline{w}_1, \overline{w}_2,\cdots,\overline{w}_n]^T$,
  follow the dynamics
  \vspace{-0.1\squeezefactor}
  \[
  \overline{w}^{(t+1)} = (W \circ Q)\overline{w}^{(t)}
  \vspace{-0.1\squeezefactor}
  \]
  where `$\circ$' indicate a Hadamard product or an element-wise product.
  %(\emph{i.e.}, the )
% 
%   (with eigenvector corresponding to eigenvalue $1$ being $q_1$),
%  and suppose that 
\end{proposition}
\begin{proof}
 %Consider the path of a particle from $x_i$ to $x_j$ in $k$ time steps through the states $x_i=x_{i_0}, x_{i_1}, x_{i_2}, \cdots, x_{i_k}=x_j$.
 %The probability of such a 
 Suppose at time $t$ the distribution (proportion) of the robots in the different states is $\rho^{(t)}$.
 Likewise, suppose $w^{(t)}$ is the average weight per particle in the different states (so that the aggregated weight, $\overline{w}^{(t)}_i = w^{(t)}_i \rho^{(t)}_i$, for state $x_i$).
 At the next time step, the expected number (proportion) of robots entering sate $x_i$ from state $x_j$ is $Q_{ij} \rho^{(t)}_j$. Each of those particle, on an average will bring in weight $W_{ij} w^{(t)}_j$ into state $x_i$.
Thus, at time $t+1$, the aggregated weight in state $x_i$ will be $\overline{w}^{(t+1)}_i = \sum_j W_{ij} w^{(t)}_j Q_{ij} \rho^{(t)}_j = \sum_j (W_{ij} Q_{ij}) \overline{w}^{(t)}_j$.
%  
%  The probability that a particle in state $x_i$ transitions to state $x_j$ in the next time step is $Q_{ji}$.
%  Thus, during that transition, the expected outflux of total weight from $x_i$ is $\sum_j w^{(t)}_i \rho^{(t)}_i Q_{ji}$.
%  Likewise, the expected total influx of weight into $x_i$ is $\sum_j w^{(t)}_j W_{ij} \rho^{(t)}_j Q_{ij}$ (note that during influx the weight of a particle gets multiplied by $W_{ij}$).
%  Thus, at time $t+1$, the expected total weight in state $x_i$ will be.
%  \[
%   w^{(t+1)}_i \rho^{(t+1)}_i = w^{(t)}_i \rho^{(t)}_i - \sum_j w^{(t)}_i \rho^{(t)}_i Q_{ji} + \sum_j w^{(t)}_j W_{ij} \rho^{(t)}_j Q_{ij}
%  \]
%  But again,
%  
% That is, $\overline{w}^{(t+1)}_i $
\end{proof}

% The above proposition allows us to use an 
Based on the above, we can choose the transition probability matrix, $Q$, to be equal to the the uniform probability matrix, $U$ (which is a matrix in which all the entries are equal to $1/n$), and choose the transition weight update matrix, $W$, to be equal to $n P$.
% In doing so, the matr
In doing so, the dynamics of the aggregated weight vector, $\overline{w}$, becomes indistinguishable from the earlier dynamics of the probability distribution in \eqref{eq:prob-dynamics}:
\vspace{-0.1\squeezefactor}
\begin{equation} \label{eq:particle-weight-dynamics}
 \overline{w}^{(t+1)} = (W \circ Q)\overline{w}^{(t)} = (n P \circ U)\overline{w}^{(t)} = P \overline{w}^{(t)}
 \vspace{-0.1\squeezefactor}
\end{equation}

Thus in the weighted version of the Monte-Carlo algorithm, the probability transition matrix is chosen to be uniform, but every time a robot makes a transition to state $i$ from state $j$, it multiplies its own weight (a real number maintained by the robot) by $n P_{ij}$.
% The plots in Figures 
For a large number of robots, due to \eqref{eq:particle-weight-dynamics}, this ensures that the eventual distribution of the aggregated weights,
% (proportion relative to the total weight in the entire system), 
$\overline{w}$, converges to a scalar multiple of $\pi_1$.
This is illustrated in Figures~\ref{fig:particles-weight-1}-\ref{fig:particles-weight-500}.

Just as before, the columns of the transition weight update matrix, $W = n P$, are use to define the kernel vectors -- $K_{*}$ for state away from boundary, $K_b$ for a boundary state $b$.
The overall algorithm that a particle follows in performing a transition is summarized below:

% \begin{itemize}
%  \item[i.] Detect (sense) if there is a boundary in the neighborhood. If so (say the robot is in boundary cell $b$), choose the boundary kernel $K_b$, otherwise choose the generic kernel $K_{*}$.
%  \item[ii.] Choose a state that is at location $\delta$ relative to robot's current state.
% \end{itemize}
\begin{minipage}{0.95\columnwidth}
\begin{algorithm}{Individual Robot Algorithm \newline in Weighted Monte-Carlo} \label{alg:individual-robot}
 \newalgline & \alglinecontent{Detect (sense) if there is a boundary in the neighborhood. If so (say the robot is in boundary cell $b$), choose the boundary kernel $K_b$, otherwise choose the generic kernel $K_{*}$.} \\
 \newalgline & \alglinecontent{Choose an action that will transition the robot to a location $\delta$ relative to robot's current state.} \\
 \newalgline & \alglinecontent{Transition using the chosen action. If successful, multiply own weight with kernel entry $k_\delta$ (which is zero if the action is larger than the radius $r$ of the kernel).} \\
%  \newalgline & \hspace{1.5em} \alglinecontent{If failure, try a different transition ().}
 \newalgline & \alglinecontent{Communicate with other robots inside the new state and compute average weight of the robots in that state. Set own weight to the average value.}
\end{algorithm}
\end{minipage}
\vspace{0.1in}

The last step (average weight computation within a state) is necessary to smoothen noise and ensure that the dynamics \eqref{eq:particle-weight-dynamics} of the weights hold.

\emph{Remark on local nature of the algorithm:} Although in the weighted version of the Monte-Carlo algorithm robots do not transition according to the Markov chain's transition probability matrix, $P$, as far as the aggregated weights are concerned, its dynamics is indistinguishable from the dynamics of robot distribution in the unweighted case.
Furthermore, even though the robots can make larger transitions (according to the uniform transition probability matrix, $U$), the \emph{local nature} of the Monte-Carlo still holds since the transition weight update is non-zero only if the robot transitions to one of the \emph{neighbors} (corresponding to the non-zero elements in the kernel) -- longer jumps to distant states do not require the robot to sense its surroundings (for kernel determination) or make non-trivial updates to its weights -- it simply multiplies its weight by zero (\emph{i.e.} set it to zero).

The advantage of the weighted Monte-Carlo over the unweighted version is that the weights in the weight matrix, $W$, need not in general satisfy the properties of a Stochastic matrix. In fact one can consider a weight matrix, the columns of which do not add up to unity, and some of the entries can even be negative. None of these violate Proposition~\ref{prop:weight-dynamics}, and hence the dynamics of the aggregated weight vector, \eqref{eq:particle-weight-dynamics}, still holds true. We will exploit this key observation in the next section.

% \todo{Explain!!}

\mysection{Harmonic Attractor Dynamics}

As described earlier, the stable/attractor eigenvector of the dynamics of the stochastic matrix, $P$, and hence the weight matrix $W = nP$, is $\pi_1$ -- the eigenvector corresponding to eigenvalue $1$ of $P$.
We now construct a different set of matrices whose attractors are the other harmonics of the Markov Chain (\emph{i.e.}, the other eigenvectors of $P$).
\begin{proposition}[Harmonic Attractor Dynamics]
%  Suppose $f$ is a polynomial such that 
 Suppose $P$ is a stochastic matrix with eigenvalues $\lambda_1, \lambda_2, \cdots, \lambda_n$ (with $\lambda_1=1, ~|\lambda_i| < 1, i\neq 1$) and corresponding eigenvectors $\pi_1,\pi_2,\cdots,\pi_n$.
%  Define $\Delt_a = \max()$
 If $f(u)$ is a polynomial in $u$ with $f(0)=1$ and $|f(\lambda_i - \lambda_a)| < 1, ~\forall i\neq a$, then
 the eigenvectors of $M_a = f(P - \lambda_a I)$ are the same eigenvectors $\pi_1,\pi_2,\cdots,\pi_n$,
 %and the corresponding eigenvalues are $\lambda_a=1, ~|\lambda_i| < 1, i\neq 1$
 with eigenvalue corresponding to $\pi_a$ equal to $1$, and the magnitudes of all the other eigenvalues less than $1$.
 Thus the dynamics $v^{(t+1)} = M_a v^{(t)}$ always converges to $\pi_a$.
\end{proposition}
\begin{proof}
%  The proof follows from the definition of matrix exponent and matrix power.
Since $\pi_i$ is an eigenvector of $P$ with eigenvalue $\lambda_i$, we have 
\[ \begin{array}{l}
  (P - \lambda_a I)^m ~\pi_i =  (P - \lambda_a I)^{m-1} (P\pi_i - \lambda_a I \pi_i) \\
  =  (P - \lambda_a I)^{m-1} (\lambda_i \pi_i - \lambda_a I \pi_i) = (\lambda_i - \lambda_a) (P - \lambda_a I)^{m-1} ~\pi_i \\
  = (\lambda_i - \lambda_a)^2 (P - \lambda_a I)^{m-2} ~\pi_i = \cdots = (\lambda_i - \lambda_a)^m ~\pi_i
 \end{array} \]
Since $f$ is a polynomial, we thus have $f(P - \lambda_a I) \pi_i = f(\lambda_i - \lambda_a) \pi_i$.
Thus, $\pi_i, i=1,2,\cdots,n$ are eigenvectors of $f(P - \lambda_a I)$ with eigenvalues $f(\lambda_i - \lambda_a) $. %This concludes the proof.
\end{proof}

% \vspace{0.1in}
\noindent Examples of such polynomials:
\vspace{-0.2\squeezefactor}
\begin{eqnarray} 
 \text{Second order:} & f(u) = 1 - \frac{\beta + 1}{\Delta_a^2} u^2 \label{eq:f-2nd}  \vspace{-0.2\squeezefactor} \\
 \text{Fourth order:} & f(u) = 1 - \frac{3(\beta+1)}{\Delta_a^2} u^2 + \frac{2(\beta+1)}{\Delta_a^4} u^4 \label{eq:f-4th} \vspace{-0.2\squeezefactor}
\end{eqnarray}
\noindent where $\Delta_a = \max_i |\lambda_i - \lambda_a|$ and $|\beta|<1$ is a parameter.
We can even choose to have different $f$ for computing the different $M_a$ (for example, $f(u) = 1 + u$ is sufficient when $a=1$ since $-2<\lambda_i - \lambda_1 < 0, ~i\neq 1$).\footnote{Non-polynomial examples include $e^{(P-\lambda_a I)^2}$, but they break the \emph{local nature} of the Monte-Carlo algorithm, as discussed later.}

% One can indeed choose other definitions of $M_a$ with similar properties.
% Suppose, for a matrix with singular value decomposition, $H = U ~\text{diag}(\sigma_1,\sigma_2, \cdots,\sigma_n)~ V^\dagger$, we define $\|H\| = U ~\text{diag}(|\sigma_1|,|\sigma_2|, \cdots,|\sigma_n|)~ V^\dagger$.
% Then the matrix,
% \begin{equation} \label{eq:Ma-def2}
%  M_a = \exp \left(-\|P - \lambda_a I\|^\alpha \right)
% \end{equation}
% for any $\alpha>0$, once again has $\pi_1,\pi_2,\cdots,\pi_n$ as eigenvectors, and corresponding eigenvalues are $e^{-|\lambda_i - \lambda_a|^\alpha}, i=1,2,\cdots,n$. The definition of $M_a$ in \eqref{eq:Ma-def2} agrees with \eqref{eq:Ma-def1} when $\alpha=2$.

From a convergence rate and numerical stability standpoint, 
% we found that for Monte-Carlo algorithms (described next) 
convergence is attained faster in the dynamics of $M_a$ when the magnitudes of eigenvalues $f(\lambda_i - \lambda_a)$ are farther from $1$ (since then the components of $\overline{w}$ along the eigenvectors other than $\pi_a$ will decay faster).
Thus, for computing $M_a$ using a general order-$r$ polynomial, $f(u) = 1 + \kappa_1 u + \kappa_2 u^2 + \cdots + \kappa_r u^r$, %\todo{optimization.}
the coefficients, $\kappa = \{\kappa_1, \kappa_2, \cdots, \kappa_r\}$, can be computed by solving the following optimization problem
\vspace{-0.2\squeezefactor}
\begin{eqnarray} \label{eq:optimization}
 \kappa^{*} & = & \argmin_{\kappa} \left( \max_{i\neq a} ~f(\lambda_i - \lambda_a) \right) \nonumber \\
  & & \text{s.t.} \quad -\beta \leq f(\lambda_i - \lambda_a) < 1-\epsilon \vspace{-0.2\squeezefactor}
\end{eqnarray}
where $-1<\beta \leq 1$ is a parameter and $0 < \epsilon \ll 1$ is a very small positive number.
This is a nonlinear optimization problem, and can be solved using sequential quadratic programming~\cite{nocedal2006numerical}. For all practical purposes we choose $\beta=0$ to ensure that all the eigenvalues of $M_a$ are positive (in order to avoid oscillations).

\mysubsection{Monte-Carlo with $M_a$ as Transition Weight Update Matrix}

\begin{figure*} 
    \begin{tabular}{cc}
     \imagetop{\begin{tabular}{c}
     	\subfloat[The harmonic/eigenvector, $\pi_5$, of the stochastic matrix, $P$, in Figure~\ref{fig:p-matrix-20}. ($L^1$-normalized)]{\includegraphics[width=0.28\textwidth, trim=60 200 60 210, clip=true]{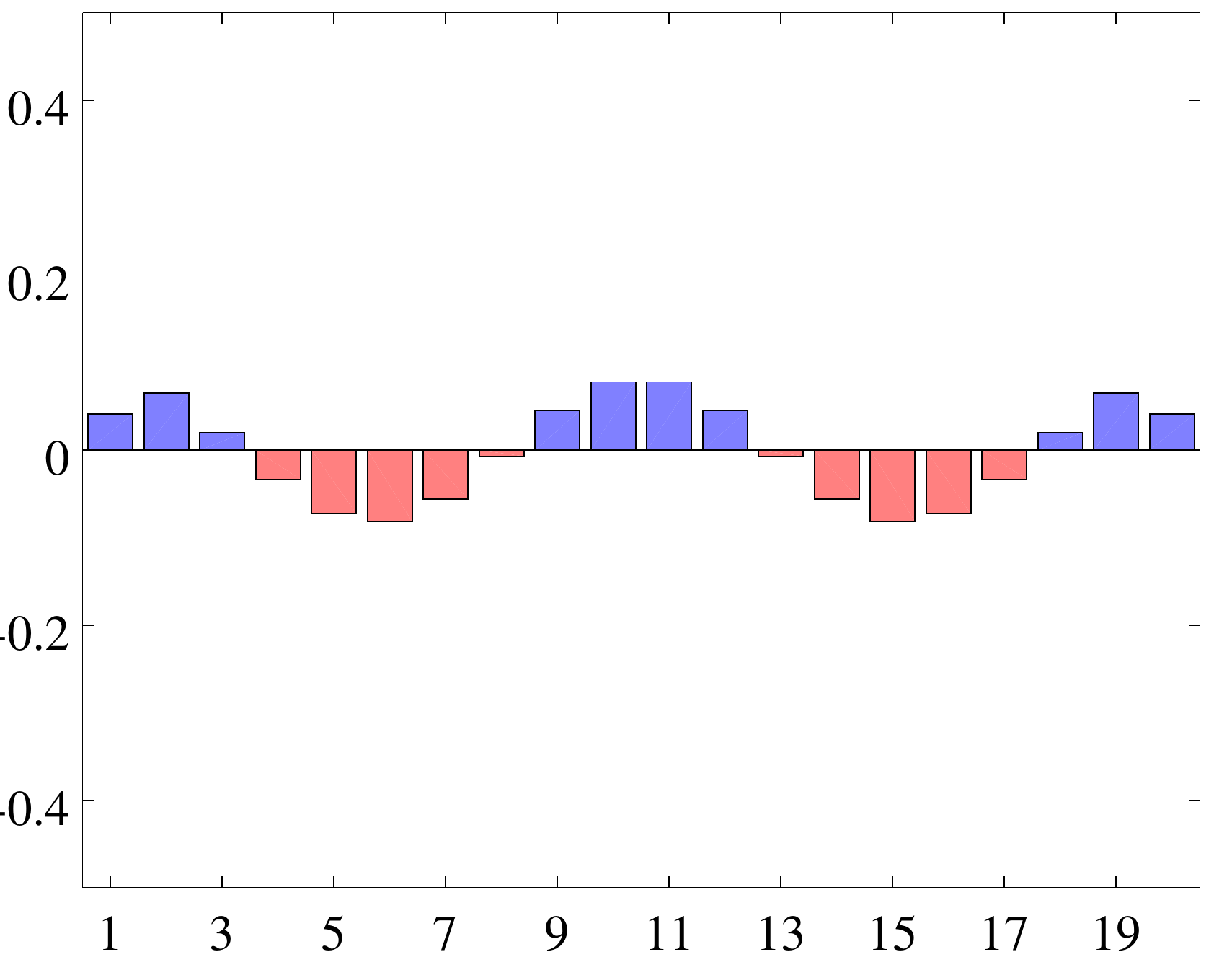} \label{fig:pi5-20}} \\
	\subfloat[The matrix $M_5$ computed using a $4^{\text{th}}$ order polynomial, $f$. \emph{Blue:} positive values, \emph{red:} negative values, \emph{white:} zero. Note that the \emph{radius} of the kernel is $4$.]{\hspace{0.15in}\includegraphics[width=0.3\textwidth, trim=0 0 0 0, clip=true]{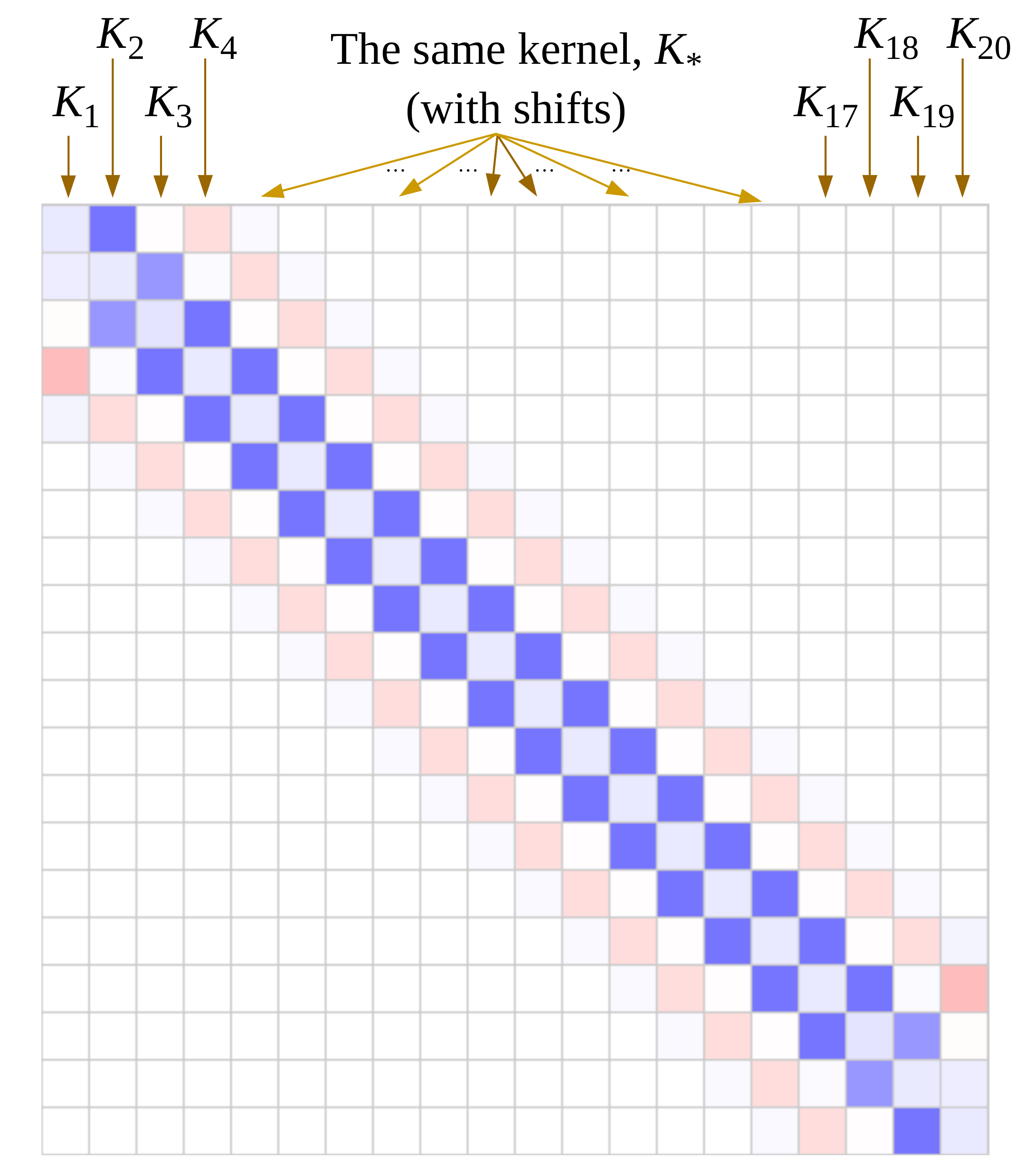}\hspace{0.1in} \label{fig:m5-matrix-20}}
     \end{tabular} } \hspace{-0.2in}
     & 
     \imagetop{\begin{tabular}{cc}
 	\subfloat[Aggregated weights at $t=1$.]{\includegraphics[width=0.3\textwidth,height=0.2\textwidth, trim=50 340 50 233, clip=true]{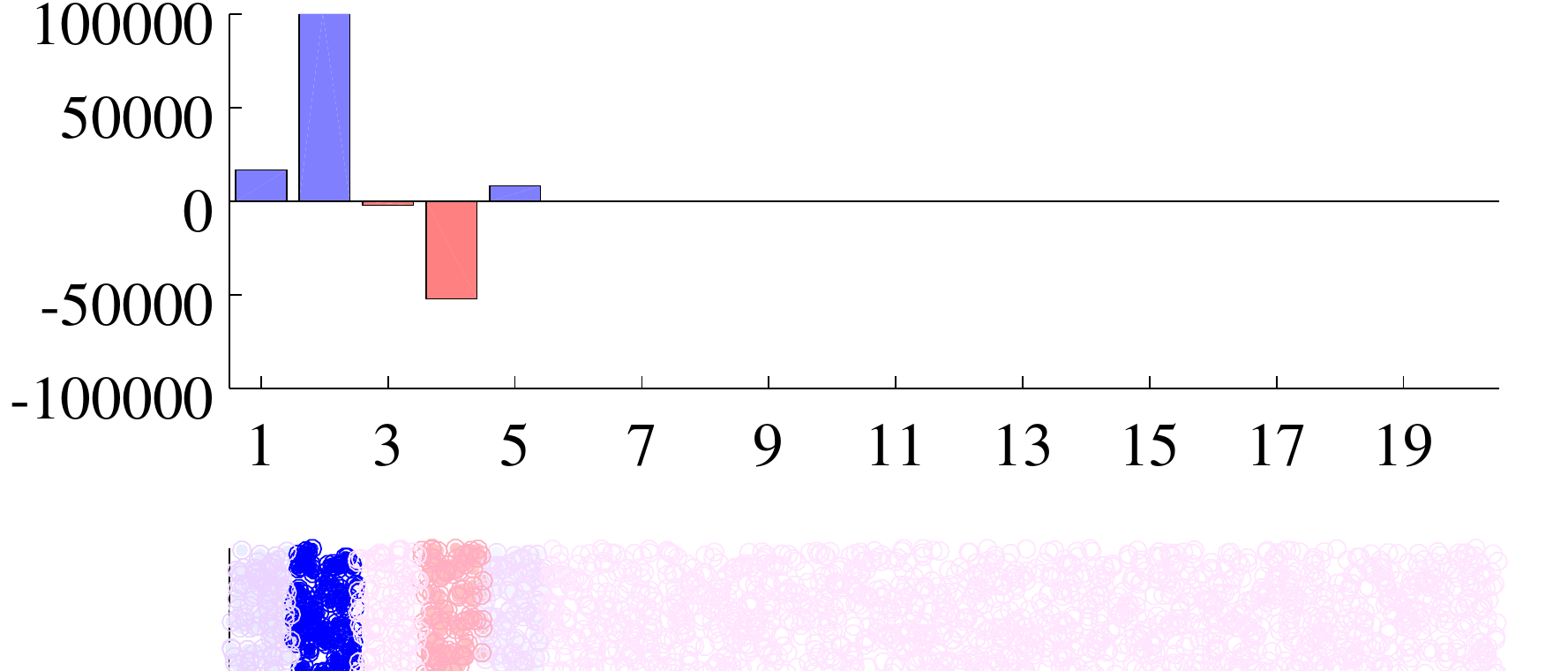} \label{fig:particles-weighted-mDynamics-1}} &
 	\subfloat[Aggregated weights at $t=20$.]{\includegraphics[width=0.3\textwidth,height=0.2\textwidth, trim=50 340 50 233, clip=true]{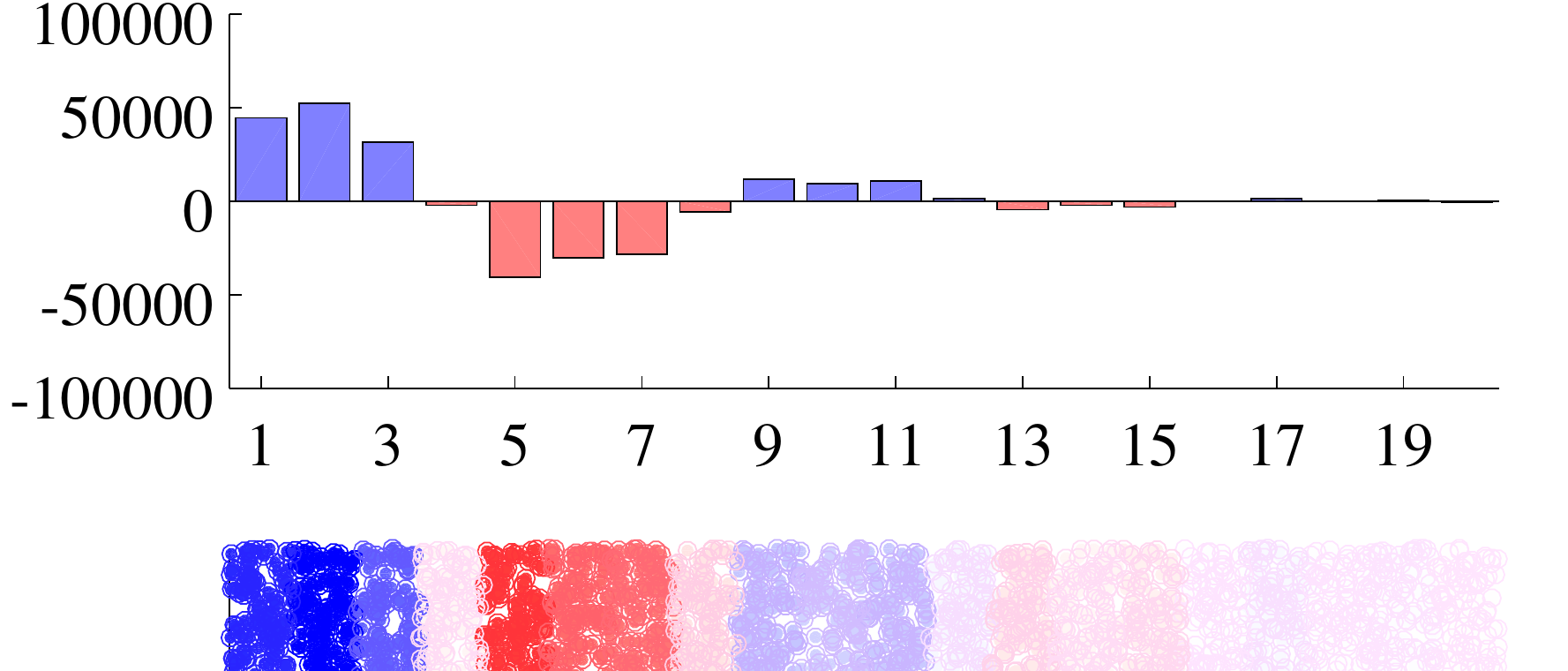} \label{fig:particles-weighted-mDynamics-20}} \vspace{0.1in} \\
 	\subfloat[Aggregated weights at $t=50$.]{\includegraphics[width=0.3\textwidth,height=0.2\textwidth, trim=50 340 50 233, clip=true]{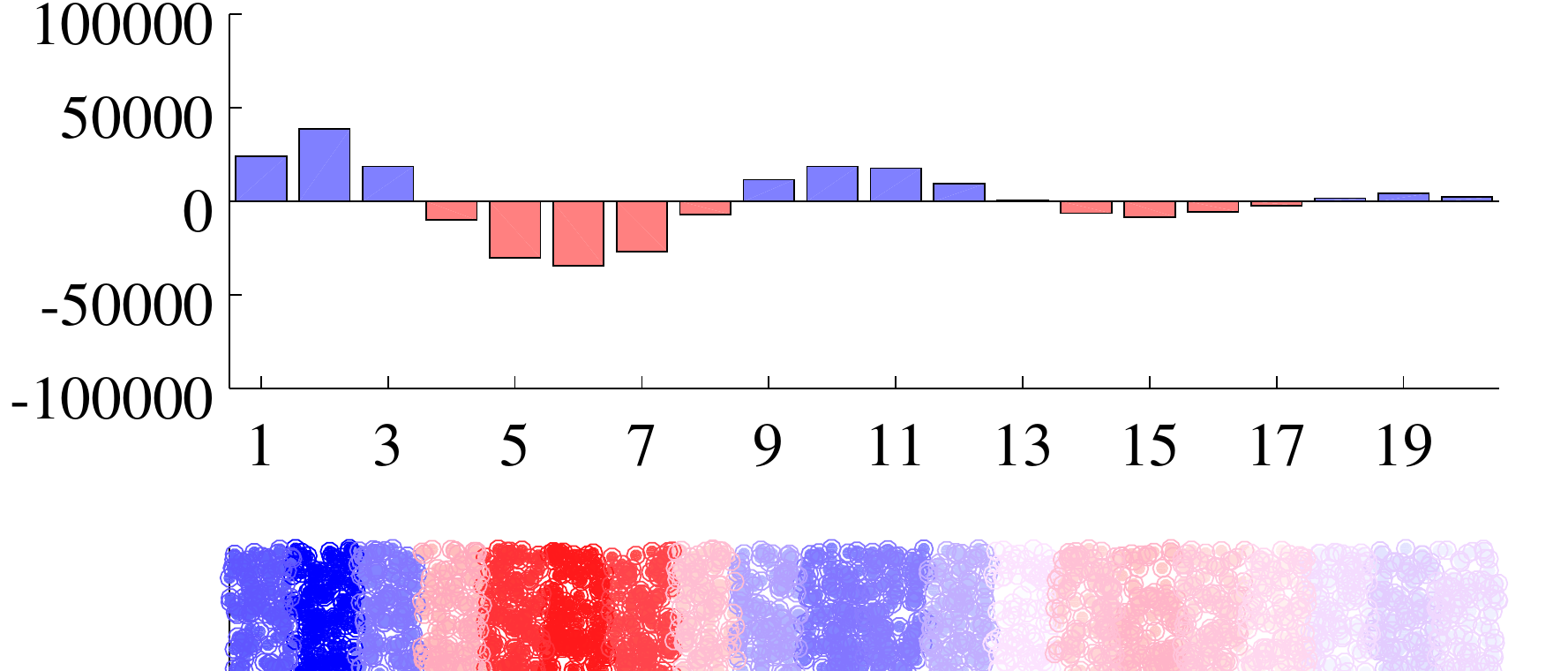} \label{fig:particles-weighted-mDynamics-50}} &
 	\subfloat[Aggregated weights at $t=200$.]{\includegraphics[width=0.3\textwidth,height=0.2\textwidth, trim=50 340 50 233, clip=true]{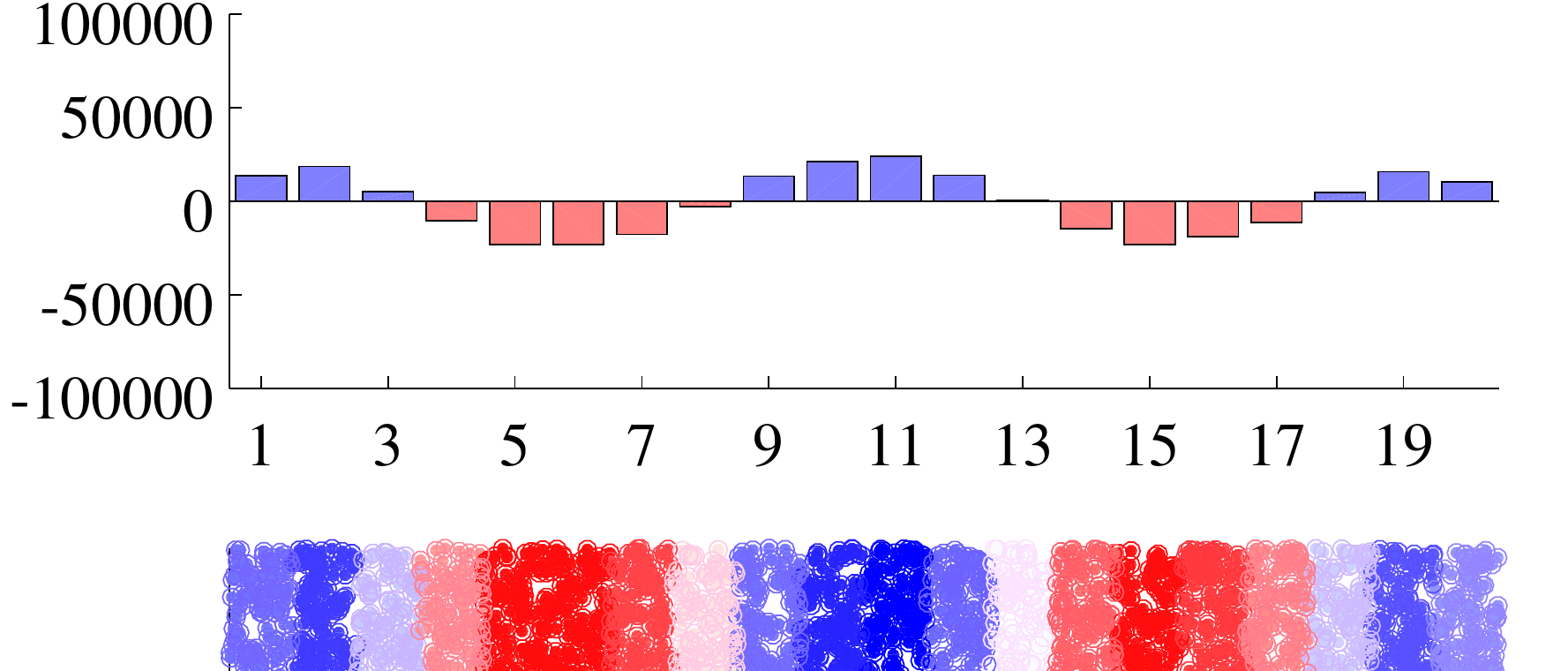} \label{fig:particles-weighted-mDynamics-200}} \vspace{0.1in}
        \\ \multicolumn{2}{c}{
	  \begin{minipage}{0.62\textwidth} \small (c-f): Monte Carlo algorithm using $200000$ robots with weights, all starting at $x=1$ at $t=0$. Robots transition according to uniform transition probability matrix, $U$. But they update their weights according to transition weight update matrix $W = n M_5$. The bar plot shows the aggregated weight vector $\overline{w}$. Each dot below represent $100$ robots, with the color of the dots indicating the weights (\emph{red:} negative, \emph{blue:} positive, \emph{lighter:} lower magnitude, \emph{darker:} higher magnitude). The aggregated weights (when normalized by the total weight across all states) eventually converge to to $\pi_5$. \end{minipage} }
     \end{tabular} }
    \end{tabular}
    \mycaption{The dynamics of $M_5$ (computed using \eqref{eq:f-4th}, with $\beta=0.7$) has the harmonic $\pi_5$ as attractor.} \label{fig:1d-harmonic5-dynamics}
\end{figure*}

If we use uniform transition probabilities, $U$, as before, and use $W = n M_a$ as the transition weight update matrix, then by Proposition~\ref{prop:weight-dynamics} the dynamics of the vector of aggregated weights, $\overline{w}$, is governed by
\vspace{-0.2\squeezefactor}
\[
 \overline{w}^{(t+1)} = (n M_a \circ U)\overline{w}^{(t)} = M_a \overline{w}^{(t)}
 \vspace{-0.2\squeezefactor}
\]
This implies that the aggregated weight vector, $\overline{w}$, will converge to (a scalar multiple of) the harmonic $\pi_a$. This is illustrated in Figure~\ref{fig:1d-harmonic5-dynamics}, where we choose the same Markov chain with $20$ states as in Figure~\ref{fig:1d-markov-p-20}, but the dynamics is that of $M_5$.

\mysubsection{Local Nature of Weighted Monte-Carlo with $M_a$ Dynamics}

The advantage of having a $r^{\text{th}}$ order polynomial for $f$ is that the resulting matrix, $M_a = f(P - \lambda_a I)$, has non-zero entries for transitions within a \emph{radius} of $r$ of the robot's current state. Since the highest power of $P$ in $f(P - \lambda_a I)$ is $r$, the element in the position $(i,j)$ in matrix $M_a$ can be non-zero if only if states $x_i$ can be reached from state $x_j$ in at most $r$ hops in the Markov Chain whose transition probability matrix is $P$. 

Thus the kernel $K_{*}$ of the dynamics of $M_a$ has a radius of at most $2r + 1$, which is the same kernel for all the states except for boundary states, $b$, where the kernel is $K_b$. Thus, if the order of the polynomial $f$ is $r$, then the robot needs to be able to sense within a \emph{disk} of radius $r$ around itself, in oder to determine if it needs to use a boundary kernel (if a boundary state is sensed in that disk), otherwise it uses the generic kernel $K_{*}$ to update weight upon taking an action. [Note: Here, by ``disk of radius $r$'', we mean the states that are within $r$ hops in the Markov chain.]

%% TODO: Remove noise.

\mysubsection{Conservation Principal and Initial Weight Choice}

% \todo{Conservation law? $L^1$ norm conserved?}
As in case of both the dynamics of $P$ and $M_a$, the aggregated weight vector $\overline{w}$ converges to a scalar multiple of the dynamics' attractor harmonic. This is because the $1$-dimensional vector space spanned by the attractor harmonic constitutes the eigenspace corresponding to the eigenvalue of $1$.
The following proposition gives the recipe for choosing the right initial weights for the particles in order to ensure that the final converged weight is as desired.
\begin{proposition}[Conservation of Weight Projection] \label{prop:conservation}
 Suppose $M_a$ is a matrix with eigenvectors $\pi_1,\pi_2,\cdots,\pi_n$ and corresponding eigenvalues $\lambda_1, \lambda_2, \cdots, \lambda_n$ such that $\lambda_a = 1$ and $|\lambda_i| < 1,~i\neq a$.
 The dynamics $\overline{w}^{(t+1)} = M_a \overline{w}^{(t)}$ converges to $\lim_{t\rightarrow\infty} \overline{w}^{(t)} = c_a \pi_a$, where
 \vspace{-0.2\squeezefactor}
 \[
  c_a = \frac{\overline{w}^{(0)} \cdot \pi_a}{\|\pi_a\|^2}
  \vspace{-0.2\squeezefactor}
 \]
where, `$\cdot$' is the vector dot product, and `$\|\cdot\|$' is the vector $2$-norm.
Furthermore, the projection of $\overline{w}^{(t)}$ in the eigenspace spanned by $\pi_a$ remains conserved. That is, $\overline{w}^{(0)} \cdot \pi_a = \overline{w}^{(1)} \cdot \pi_a = \cdots = \overline{w}^{(t)} \cdot \pi_a = \cdots$. [Note: Here we do not assume that the eigenvectors are $L^2$ normalized.]
\end{proposition}
\begin{proof}
 The proof follows by decomposing the initial aggregated weight vector into components along the different eigen-directions: $\overline{w}^{(0)} = c_1 \pi_1 + c_2 \pi_2 + \cdots + c_a \pi_a + \cdots + c_n \pi_n$, and thus observing that $\overline{w}^{(t)} = c_1 \lambda_1^t \pi_1 + c_2 \lambda_2^t \pi_2 + \cdots + c_a \pi_a + \cdots + c_n \lambda_n^t  \pi_n$.
\end{proof}

Conversely, if the desired final/converged aggregated weight vector is $c_a \pi_a$, then we need to choose the initial vector $\overline{w}^{(0)}$ such that $\overline{w}^{(0)}\cdot \pi_a = c_a \|\pi_a\|^2$.

\mysection{$2$-dimensional Environments and Structure Construction}

For the purpose of this paper, a $2$-dimensional environment is discretized into an uniform grid, and the Markov chain is 
% descrbed transition to neighbors 
such that for a cell with all available neighbors (Figure~\ref{fig:markov-2d}(i)), the probabilities of vertical or horizontal transitions are $0.14$, the probabilities of diagonal transitions are $0.1$, and the probability of transitioning to the same state is $0.04$.
In case one or more of the neighboring states are unavailable (obstacles or boundary), the freed-up probability gets uniformly distributed among the remaining neighbors (Figure~\ref{fig:markov-2d}(ii)).

\begin{figure}[h]
 \begin{centering}
  \includegraphics[width=0.45\textwidth, trim=0 0 0 0, clip=true]{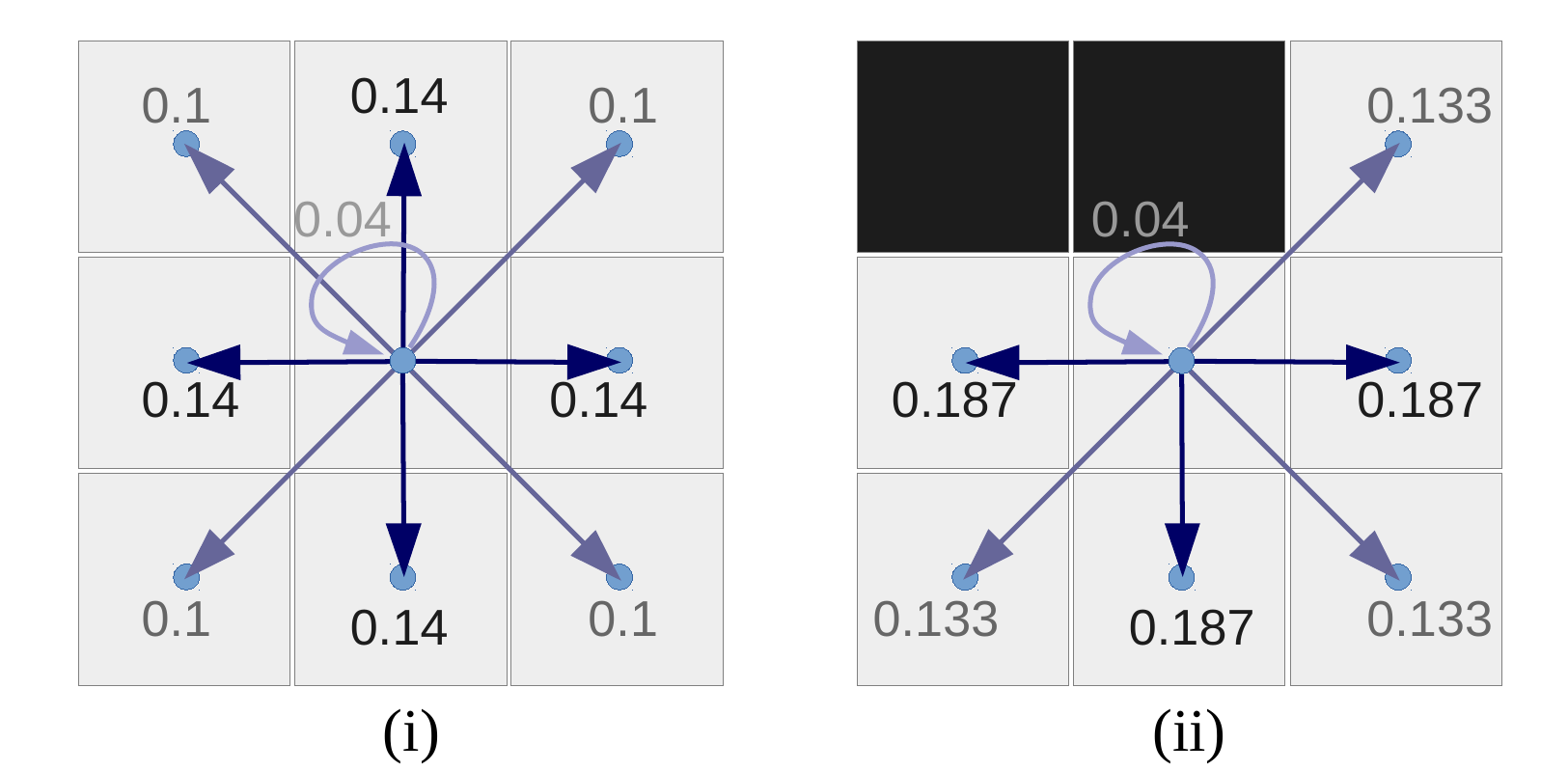} 
 \end{centering}
 \mycaption{Markov chain for a grid representation of a $2$-dimensional space. \emph{Left:} When all neighbors are free. \emph{Right:} An example where some of the neighbors are occupied.} \label{fig:markov-2d}
\end{figure}

% All our previous discussions remain 
As before, if there are $n$ states, the matrix $P$ is $n\times n$ with entries in it being non-zero only if it corresponds to transition between two neighboring states. The matrices of the harmonic attractor dynamics, $M_a$, assuming $f$ is an order-$r$ polynomial, has non-zero elements corresponding to transitions that are at most $r$ hops in the Markov chain.
We, in particular, choose $4^{th}$ order polynomials computed through the optimization \eqref{eq:optimization}.

\begin{figure}
\begin{centering}
 \subfloat[$29$ out of $120$ harmonics constitute a target shape (blue).]{\includegraphics[width=0.15\textwidth, trim=160 230 140 210, clip=true]{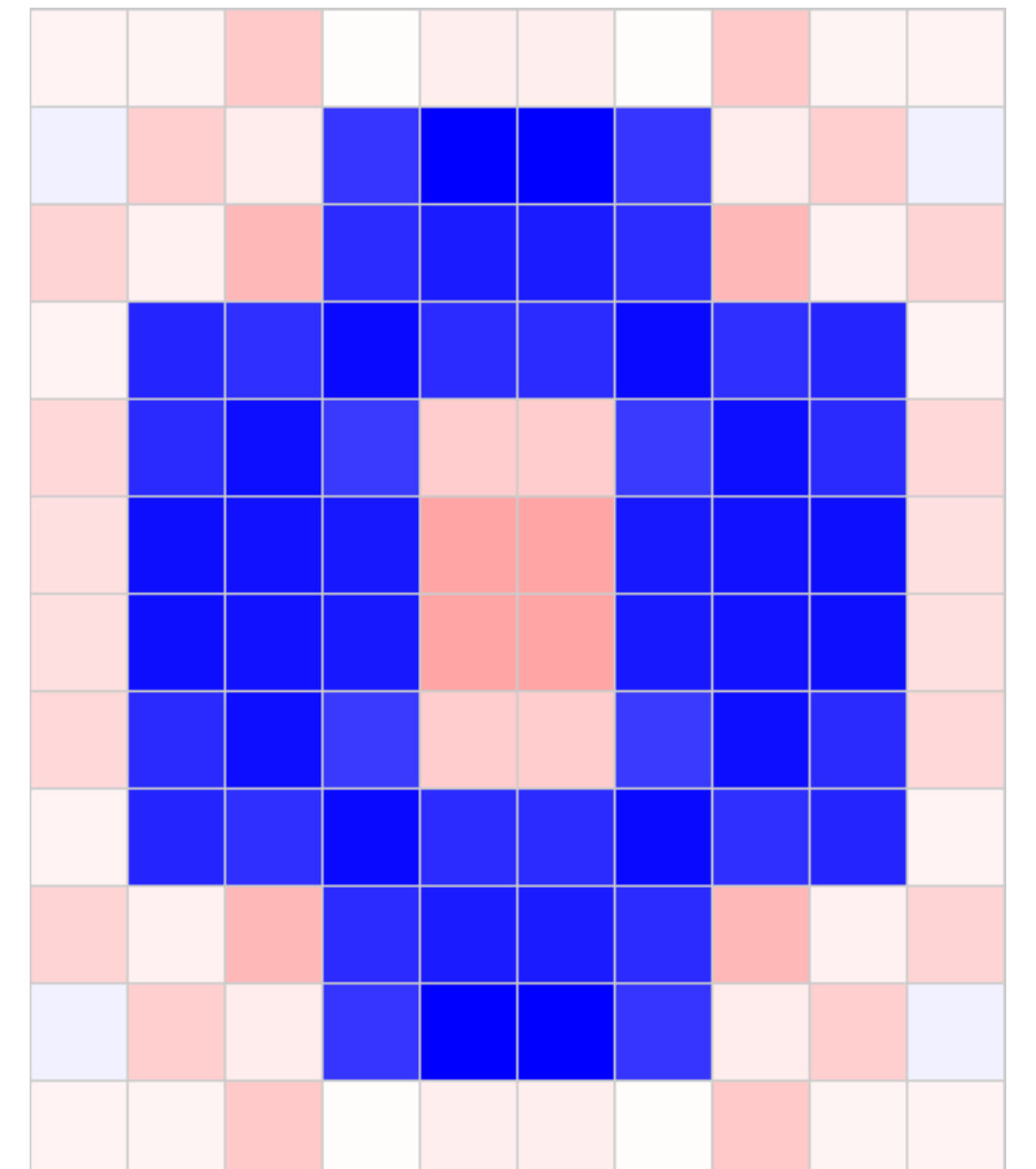} } \hspace{-0.01in}
  \subfloat[The total aggregated weight of the robots in the swarms across all the harmonics.]{\includegraphics[width=0.15\textwidth, trim=160 230 140 210, clip=true]{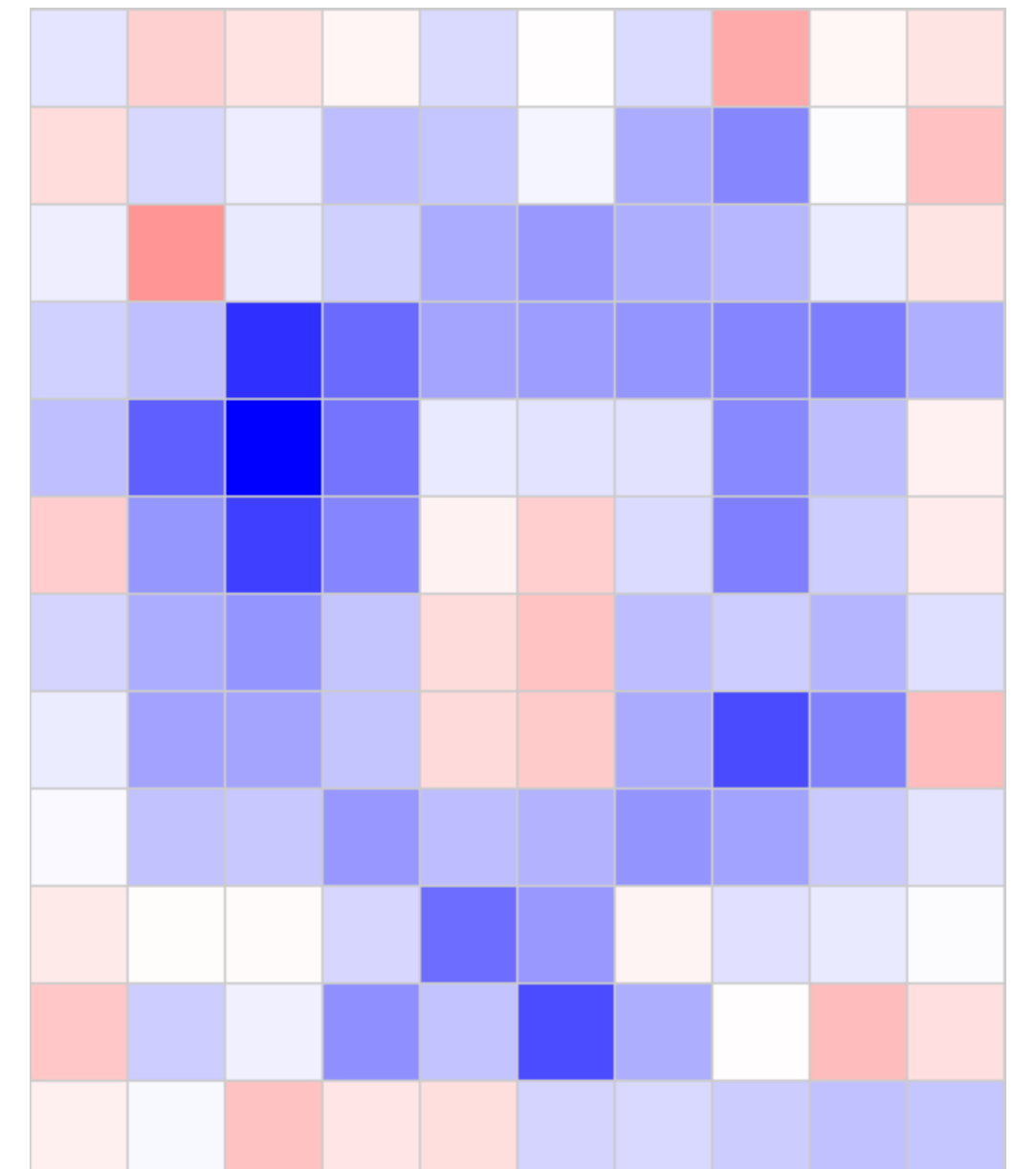} }  \hspace{-0.01in}
   \subfloat[The robots threshold the local weights to reconstruct desired structure.]{\includegraphics[width=0.15\textwidth, trim=160 230 140 210, clip=true]{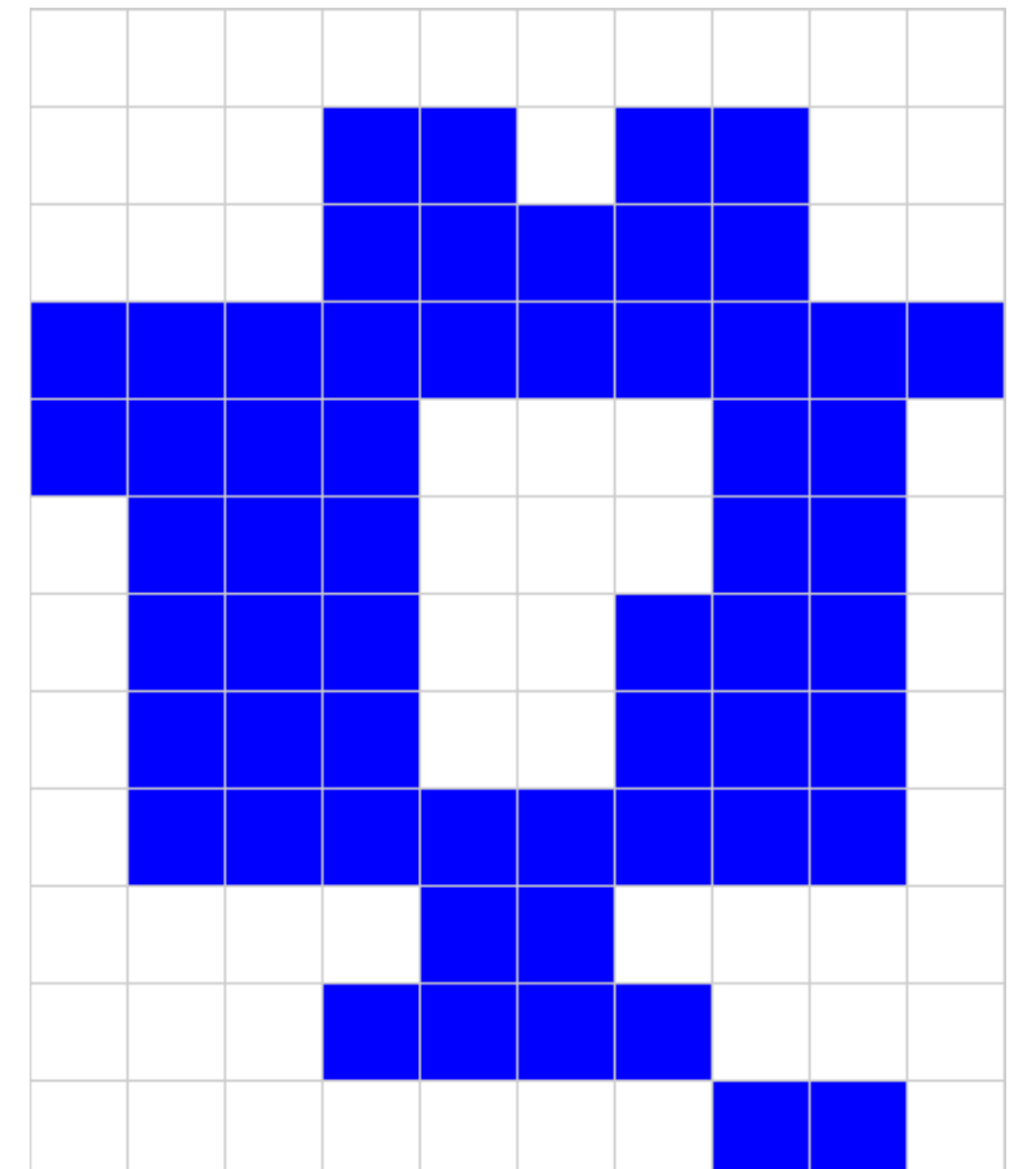} }
   \mycaption{Approximate construction of an annular shape in an obstacle-free environment using $29$ groups/swarms of robots, each constructing one of the constituent harmonics.} \label{fig:annular}
\end{centering}
\end{figure}

The overall algorithm in deploying the swarms consists of the following steps:

\begin{minipage}{0.95\columnwidth}
\begin{algorithm}{Deployment of Large Statistical Swarms For Structure Construction} \label{alg:deployment}
 \newalgline & \alglinecontent{Compute the harmonics, $\pi_1,\cdots, \pi_n$ of a given space and decompose the desired shape, $\overline{w}^{des}$, into constituent harmonics. Choose the best $p$ percentile of the harmonics (ordered by the coefficients of the $L^2$-normalized harmonics in the decomposition) to approximate $\overline{w}^{des} \simeq c_{\alpha_1} \pi_{\alpha_1} + c_{\alpha_2} \pi_{\alpha_2} + \cdots + c_{\alpha_{\ceil{np}}} \pi_{\alpha_{\ceil{np}}}  $ (top row of Figure~\ref{fig:hrmonics-monte-carlo}, Figure~\ref{fig:arrow-target-reconstruction}).} \\
 \newalgline & \alglinecontent{For each of the constituent harmonics, $\pi_{\alpha_h}$, assign a swarm, $\mathcal{S}_{\alpha_h}$, constituting of $N$ robots. The robots in $\mathcal{S}_{\alpha_h}$ are equipped with / informed of the Kernels of $M_{\alpha_h}$. Assuming all robots start at state $s$, assign initial weights of $\frac{c_{\alpha_h} \|\pi_{\alpha_h}\|^2}{N \pi_{\alpha_h, s}}$ to each robot in $\mathcal{S}_{\alpha_h}$.} \\
 \newalgline & \alglinecontent{Individual robots in $\mathcal{S}_{\alpha_h}$ follow Algorithm~\ref{alg:individual-robot}, while communicating only with the robots in $\mathcal{S}_{\alpha_h}$ that are present in its current state (columns of Figure~\ref{fig:hrmonics-monte-carlo}). Continue until convergence is attained (average aggregated weights of robots in $S_h$ stabilizes).} \\
  \newalgline & \alglinecontent{Upon attaining convergence, perform environment-wide rescaling of weights to match coefficients $c_{\alpha_h}$.}\\
  \newalgline & \alglinecontent{In each state, $i$, robots of all the swarms communicate to compute the total aggregated weights $\overline{w}_i$ (Figure~\ref{fig:arrow-rescaled-reconstruction}). Robots hold position in states where the aggregated weight is above a threshold, $\tau$, otherwise they leave the environment (Figure~\ref{fig:arrow-rescaled-reconstruction-thresholded}).}
\end{algorithm}
\end{minipage}
\vspace{0.1in}

The step $4$ (rescaling) is necessary since the aggregated weight maintained by the swarm, $\mathcal{S}_a$, encounter dissipations due to the finite number of robots in the swarm.
Thus, upon attaining convergence, $\overline{w}^{(t)}\cdot \pi_a$, do not exactly remain the same as was set at $t=0$, despite the Proposition~\ref{prop:conservation} (which holds true for the limiting case of infinite robots).
A one-time global aggregated $L^2$ norm is computed after attainment of convergence for each of the swarms, $\mathcal{S}_a$, through multi-hop neighbor communication.

\emph{Implementation:}
Our current implementation (Figures~\ref{fig:1d-harmonic5-dynamics}, \ref{fig:2d-arrow}, \ref{fig:annular}) were made using Octave, and are centralized implementations. However, being statistical swarm, each robot acts independently, with only local communication with robots in its own state, and sensing within radius $r$, thus making our algorithm massively distributable and decentralizable.

\mysection{Conclusion}
% \vspace{-0.2\squeezefactor}

In this paper we present a novel method for constructing harmonics of an environment in the spatial distribution of weights carried by a large swarm of robots using \emph{harmonic attractor dynamics}. This allowed us to decompose desired structures into constituent harmonics, and have multiple swarms of robots construct them.
% 
% As the first step, we used a discrete representation of the environment and a discrete algorithm to achieve this. Moving forward, we will extend this framework for robots navigating in continuous spaces.
% 
Our algorithm is highly local at the level of individual roots, with robots having non-zero updates to their weights only when they take action within a radius of $r$. Sensing of obstacles/boundaries is also required only within a radius of $r$.
Our initial simulation results, even with finite swarm size and noise in robot actions, show promising ability to reconstruct desired structures in an environment. %Distributed and decentralized implementation of the algorithm is within 
% \todo{.}

% Simulation TODOs: 2d environment, 2d environment with obstacles.

% \fbox{\begin{minipage}{15em}
\begin{figure*}
    \begin{tabular}{ccc}
%      	\imagetop{\subfloat[Markov chain for a grid representation of a $2$-dimensional space. \emph{Left:} When all neighbors are free. \emph{Right:} An example where some of the neighbors are occupied.]{\includegraphics[width=0.3\textwidth, trim=0 0 0 0, clip=true]{figures/markov_2d.pdf} \label{fig:markov-2d}}} %\vspace{-0.8in}
     \imagetop{\subfloat[A $12\times 10$ grid environment with obstacles (black), $95$ free cells, and a desired shape to be constructed by the swarms (blue).]{\hspace{0in} \includegraphics[width=0.21\textwidth, trim=160 210 140 210, clip=true]{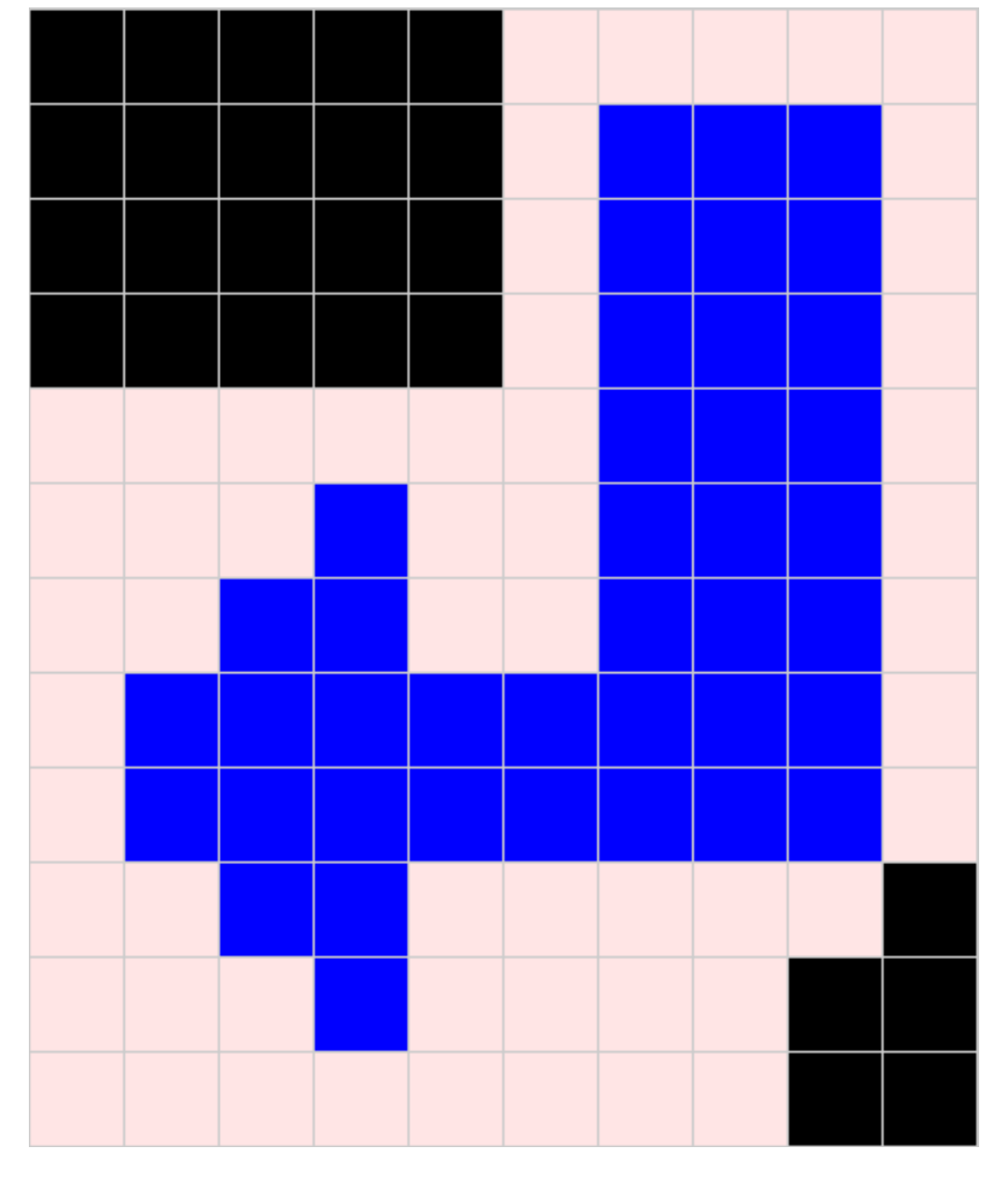} \hspace{0in} \label{fig:arrow-desired}}} \hspace{-0.2in}
     &
     \imagetop{\subfloat[Harmonics $\pi_1,\pi_2,\cdots,\pi_6$ of the environment (first $6$ out a total of $95$ in order of the magnitude of corresponding eigenvalues of $P$).]{
     \begin{tabular}{c}
     \includegraphics[width=0.1\textwidth, trim=160 210 140 210, clip=true]{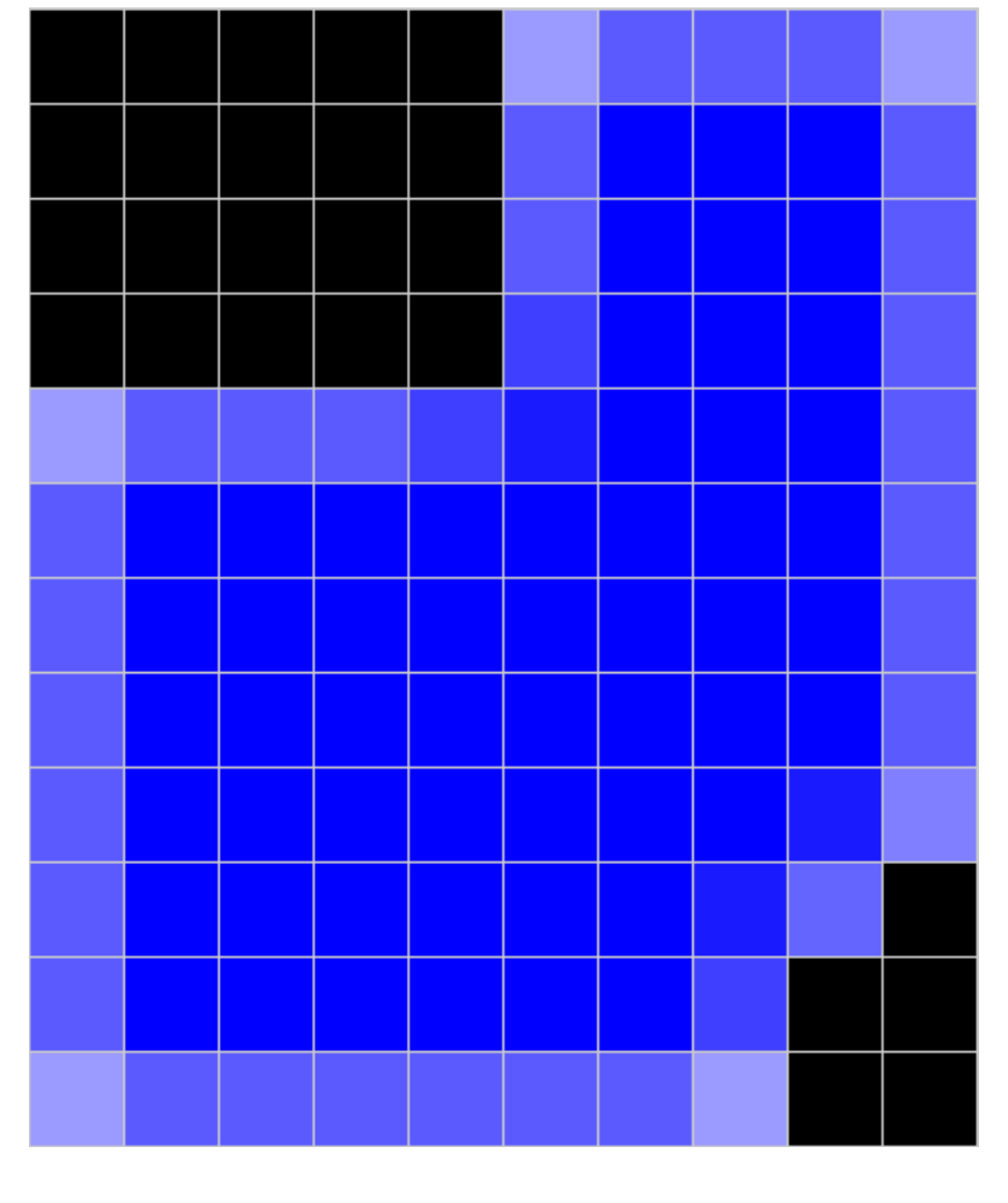}
     \includegraphics[width=0.1\textwidth, trim=160 210 140 210, clip=true]{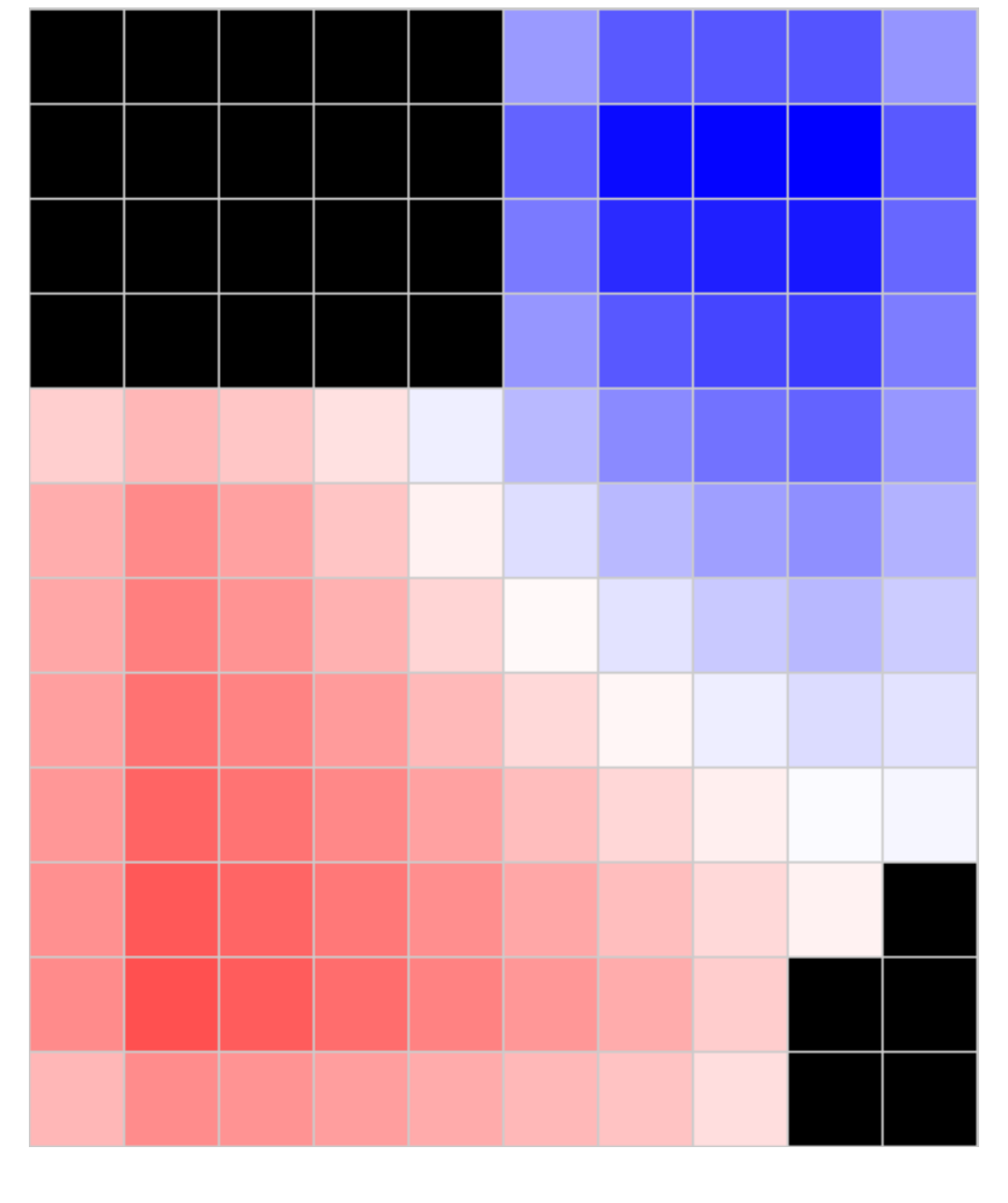}
     \includegraphics[width=0.102\textwidth, trim=160 210 140 210, clip=true]{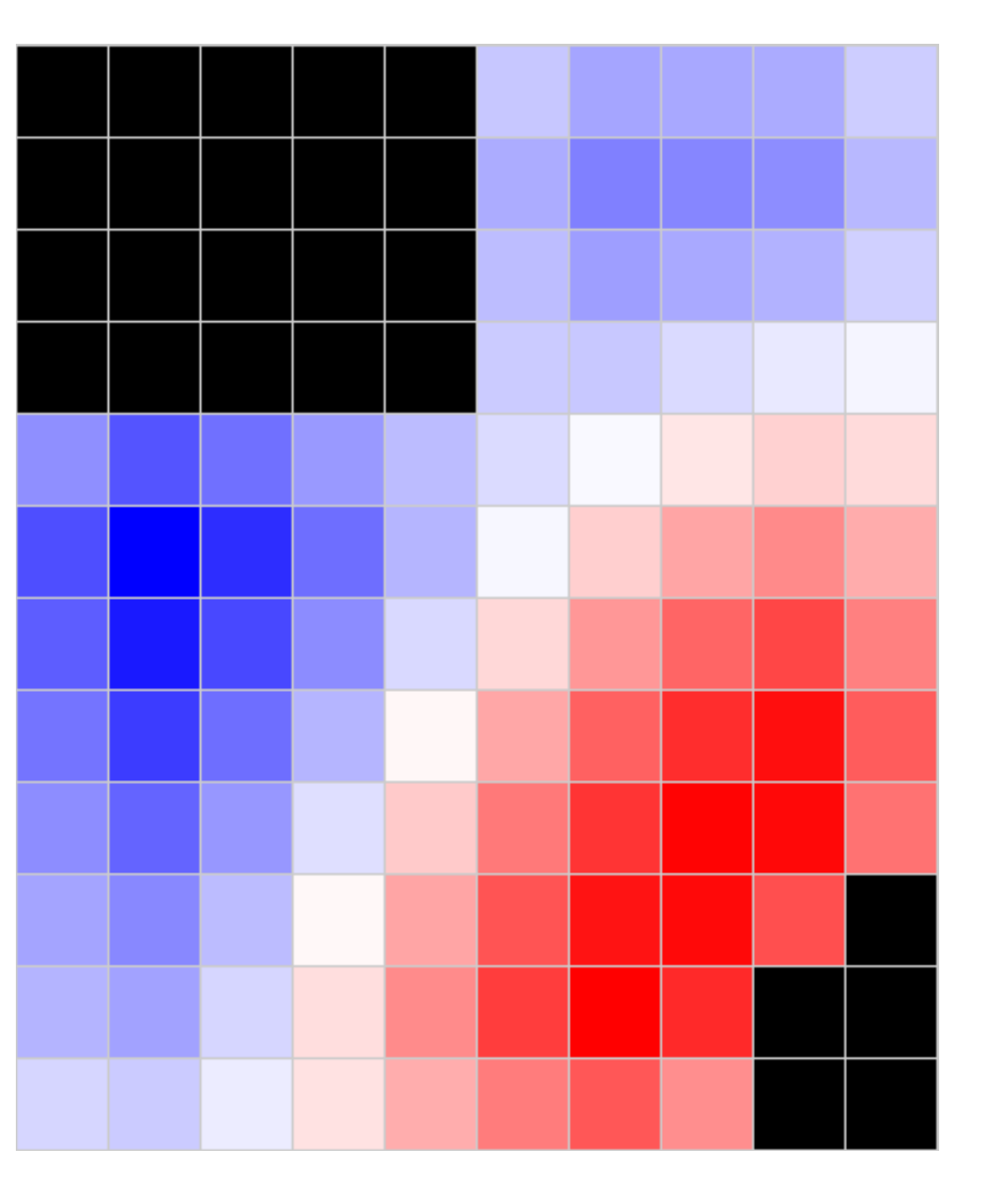} \\
     \includegraphics[width=0.1\textwidth, trim=160 210 140 210, clip=true]{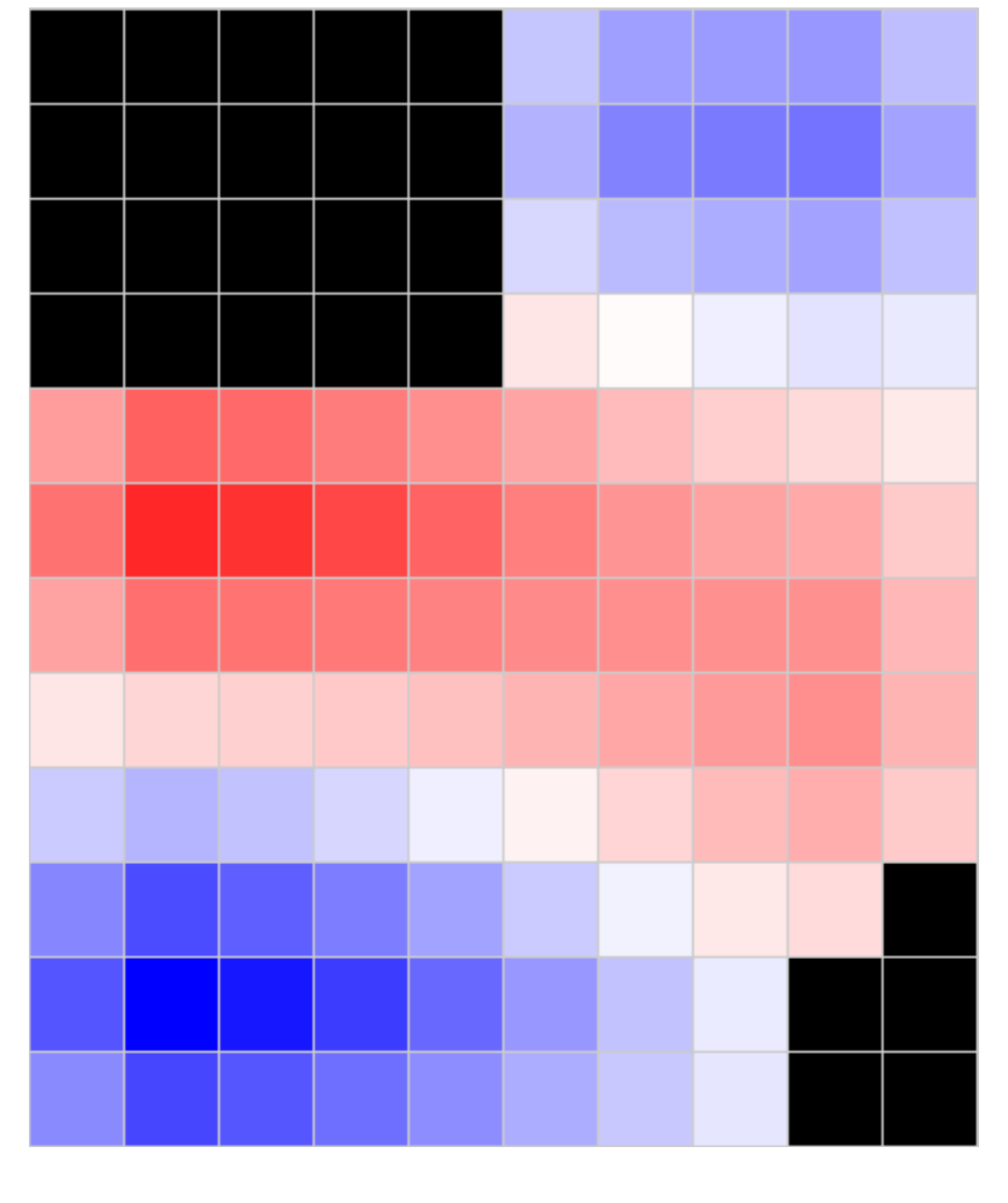}
     \includegraphics[width=0.1\textwidth, trim=160 210 140 210, clip=true]{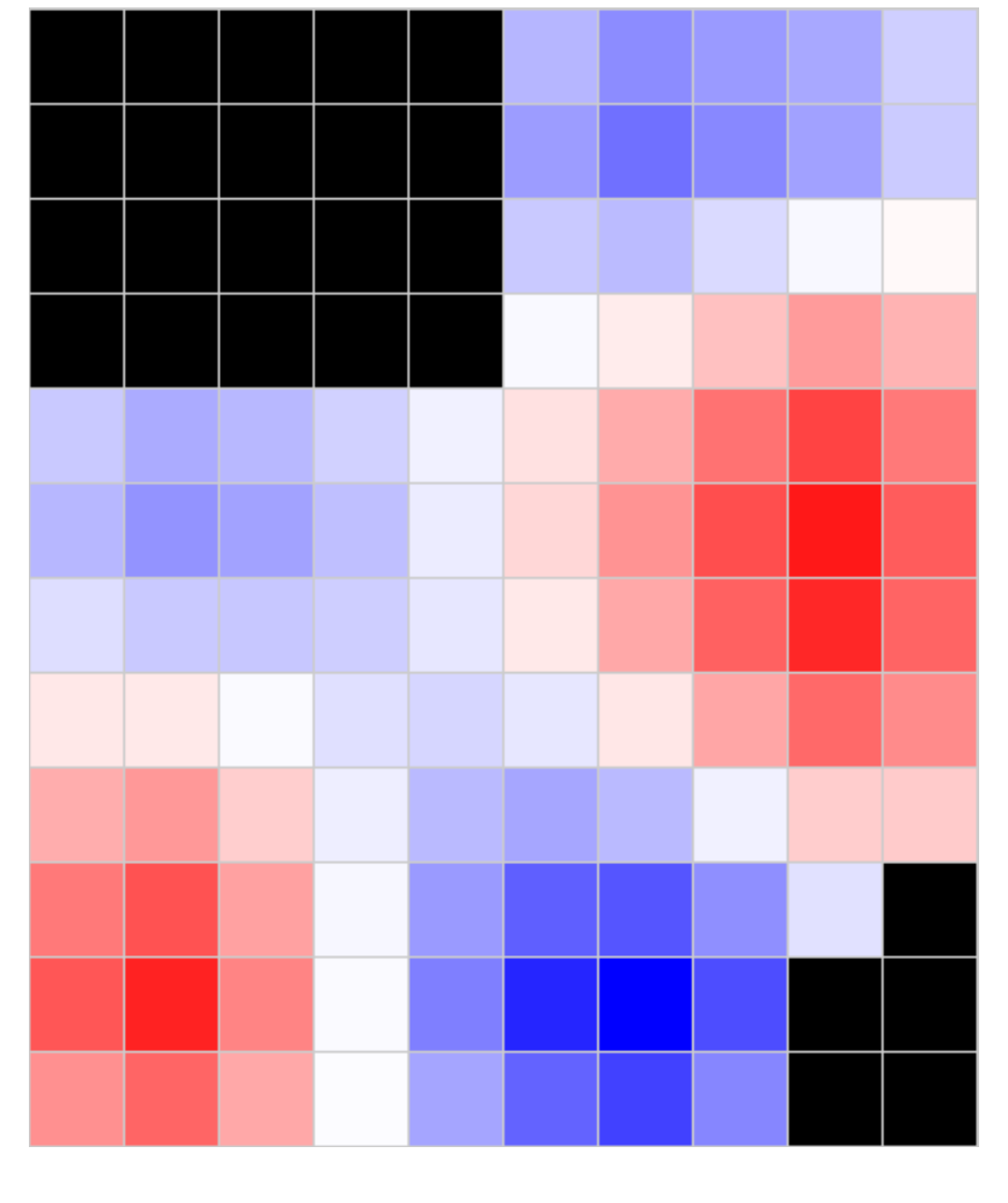}
     \includegraphics[width=0.1\textwidth, trim=160 210 140 210, clip=true]{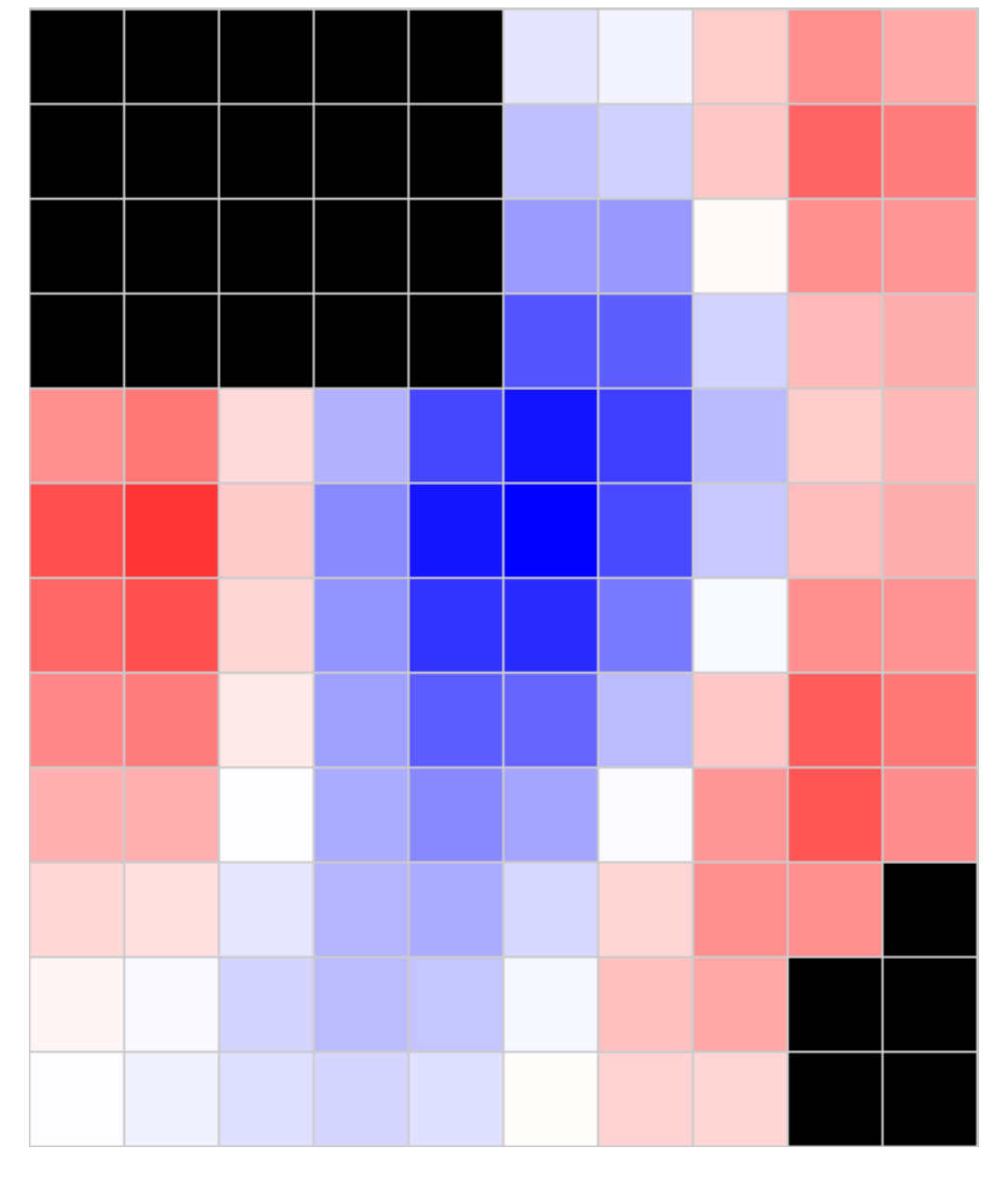}
      \end{tabular} } } \hspace{-0.2in}
      &
      \imagetop{\subfloat[The target shape constructed using a superposition of the first $24$ harmonics (out of $95$, ordered by the magnitude of coefficients of $L^2$ normalized harmonics) constituting the desired shape in Figure~\ref{fig:arrow-desired}. This is the shape that the swarms attempt to reconstruct.]{\hspace{0.4in}\includegraphics[width=0.21\textwidth, trim=160 210 140 210, clip=true]{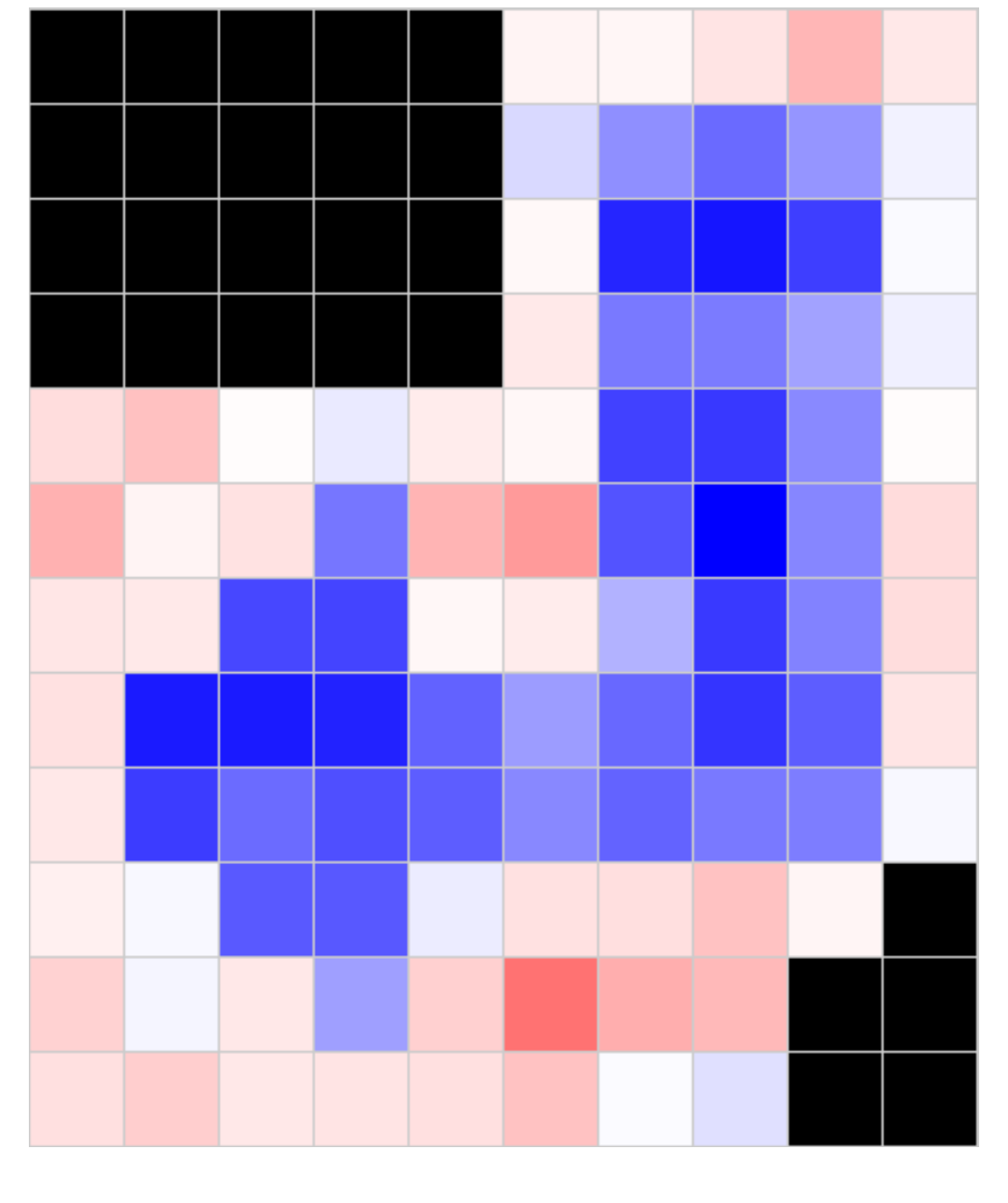} \hspace{0.4in} \label{fig:arrow-target-reconstruction}}}
      \\
%         \imagetop{\subfloat[The superposition of the first $24$ harmonics (out of $95$, ordered by the magnitude of coefficients of $L^2$ normalized harmonics) constituting the desired shape in Figure~\ref{fig:arrow-desired}.]{\includegraphics[width=0.21\textwidth, trim=160 210 140 210, clip=true]{figures/arrow_desired_with_24of95_harmonics.pdf} \label{fig:arrow-target-reconstruction}}}
%       & 
       \multicolumn{2}{c}{ \imagetop{\subfloat[The first $5$ of the constituent harmonics of Figure~\ref{fig:arrow-target-reconstruction}, and its weighted Monte-Carlo based reconstruction using $190000$ robots per swarm representing a harmonic. The first row shows the actual harmonics. The row below that shows the index, $a$, of the harmonics and the corresponding coefficients, $c_i$, in the desired shape. The rows below show the aggregated weights in robot swarm $\mathcal{S}_a$ from the dynamics of $M_a$.]{\includegraphics[width=0.65\textwidth, trim=150 900 150 60, clip=true]{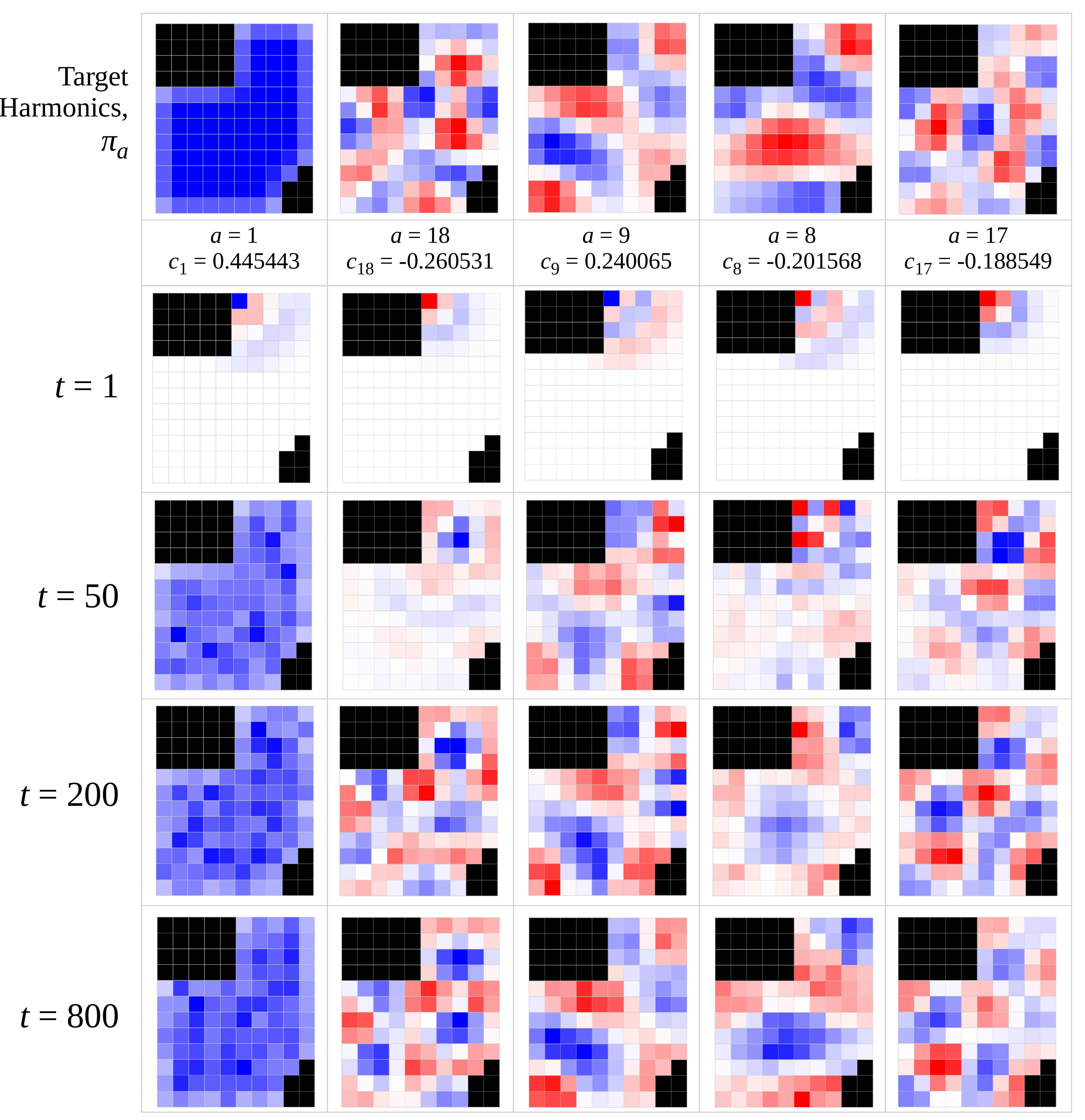} \label{fig:hrmonics-monte-carlo}} }}
       & 
        \imagetop{ \begin{tabular}{c}
         \imagetop{\subfloat[The the total aggregated weight at $t=800$ from the robot swarms constructing $24$ constituent harmonics of Figure~\ref{fig:arrow-target-reconstruction}.]{\hspace{0.2in}\includegraphics[width=0.2\textwidth, trim=150 40 130 32, clip=true]{figures/PostAnalized_reconstruction_arrow_rescaled_800.pdf} \hspace{0.2in} \label{fig:arrow-rescaled-reconstruction}}}
         \\
         \imagetop{\subfloat[After attaining convergence, robots communicate locally to threshold their weights and recreate an approximation representation of the desired Figure~\ref{arrow-desired}.]{\hspace{0.2in}\includegraphics[width=0.2\textwidth, trim=150 40 130 32, clip=true]{figures/PostAnalized_reconstruction_arrow_rescaled_thresholded_800.pdf} \hspace{0.2in} \label{fig:arrow-rescaled-reconstruction-thresholded}}}
        \end{tabular} }
     \end{tabular} 
    \mycaption{A $2$-dimensional environment and desired shape to be constructed in it. A total of $24$ large groups/swarms of robots construct the best $24$ constituent harmonics using the respective harmonic attractor dynamics.} \label{fig:2d-arrow}
\end{figure*}
% \end{minipage}}

\clearpage
\bibliographystyle{plain}
\bibliography{markov1.bib}

\end{document}